\pdfoutput=1

\documentclass{article}
\usepackage{iclr2027_conference,times}
\usepackage{amsmath,amssymb,amsthm}
\usepackage{etoolbox}

\usepackage{amsmath,amsfonts,bm}

\def\eqref#1{equation~\ref{#1}}
\def\1{\bm{1}}

\def\vk{{\bm{k}}}
\def\vl{{\bm{l}}}
\def\vm{{\bm{m}}}

\DeclareMathAlphabet{\mathsfit}{\encodingdefault}{\sfdefault}{m}{sl}
\SetMathAlphabet{\mathsfit}{bold}{\encodingdefault}{\sfdefault}{bx}{n}

\newcommand{\R}{\mathbb{R}}

\usepackage{thmtools}
\usepackage{thm-restate}
\usepackage{graphicx}
\usepackage{wrapfig}
\usepackage{caption}
\usepackage{wrapstuff}
\usepackage{xcolor}

 \newcommand{\C}{{\mathbb C}}

\newcommand{\Z}{{\mathbb Z}}

\newcommand{\NN}{{\mathbb N}}
\newcommand{\ZZ}{{\mathbb Z}}
\newcommand{\RR}{\mathbb R}

\newcommand{\Symn}{\mathcal{H}_n}

\newcommand{\Gdn}{SE(d)\times S_n}
\newcommand{\Gthree}{SE(3)\times S_n}

\newcommand{\ellmax}{\ell_{\mathrm{max}}}
\newcommand{\numax}{\nu_{\mathrm{max}}}
\newcommand{\nuread}{\nu_{\mathrm{read}}}
\newcommand{\Kmax}{K_{\mathrm{max}}}

\newcommand{\Vcent}{V_{\mathrm{cent}}}

\newcommand{\OO}{{\rm O}}
\newcommand{\SO}{{\rm SO}}
\newcommand{\Rdnmat}{\R^{d\times n}}
\newtheorem{theorem}{Theorem}
\newtheorem{proposition}[theorem]{Proposition}
\newtheorem{lemma}[theorem]{Lemma}
\newtheorem{corollary}[theorem]{Corollary}

\newtheorem{definition}[theorem]{Definition}
\newtheorem{remark}[theorem]{Remark}
\newtheorem{example}[theorem]{Example}

\newcommand{\GL}{\operatorname{GL}}

\newcommand{\twopartdef}[4]
{
	\left\{
		\begin{array}{ll}
			#1 & \mbox{if } #2 \\
			#3 & \mbox{if } #4
		\end{array}
	\right.
}

\usepackage{hyperref}
\usepackage{url}
\usepackage{cleveref}

\title{ Degree Bounds for Separating Polynomial and Neural Invariants for Point Clouds }

\author{Snir Hordan$^{1,}$\thanks{Corresponding author: \texttt{snirhordan@campus.technion.ac.il}}\hspace{0.45em} Harm Derksen$^{2}$, Mircea Petrache$^{3}$, Nadav Dym$^{1}$ \\
\normalfont $^{1}$Technion -- Israel Institute of Technology \\
\normalfont $^{2}$Northeastern University \\
\normalfont $^{3}$Pontificia Universidad Cat\'olica de Chile
}

\iclrfinalcopy

\begin{document}

\maketitle
\lhead{Preprint. Under review.}

\begin{abstract}
Neural networks that respect the symmetries of point clouds commonly compute high-order representations of the rotation group to improve their expressivity. Such models are known to be universal, and thus able to separate any pair of point clouds, when using representations of sufficiently high order, yet no precise bound on this order, as a function of the number of points, was previously known. We establish the first quadratic lower and upper degree bounds for polynomials that separate point clouds of $n$ points up to rotations and permutations. By showing Tensor Field Networks and MACE can express all invariant polynomials of a given degree with sufficient representation order, we additionally derive the first explicit, quadratic bound on the representation order sufficient for these architectures to be complete, replacing previous existence results with concrete guarantees. We further show that, under local regularity assumptions, point cloud separation can be achieved with a message-passing-based model using invariants whose degree depends only on the size of a local neighborhood, rather than on the size of the full point cloud. Finally, we provide counterexamples showing MACE must compute high order irreducible representations to separate locally regular point clouds.

\end{abstract}

\section{Introduction}
Learning functions that are invariant, or equivariant, to a group of symmetries is a central problem in geometric deep learning \citep{bronstein2021geometric}. One of the most important instances of this problem is the case of \emph{point clouds}, or finite sets of points, that efficiently represent molecules, particle systems and $3$D shapes, where the relevant symmetry group is generated by permutations of the points together with the rigid motions of the ambient space \citep{guo2019deep, han2022geometrically, pmlr-v119-sanchez-gonzalez20a}. Here the goal is to design architectures whose internal mechanisms respect the symmetries of point clouds, that is point ordering and orientation in space, and whose predictions depend only on the underlying geometry rather than on ambiguous choices of point order or choice of frame of reference \citep{schutt2017quantum, TFN}.

A natural theoretical desideratum of invariant architectures is \emph{completeness}: informally, a family of invariant functions is complete if it can separate every pair of inputs that are not related by a symmetry, that is, that do not lie in the same group orbit. This follows from completeness being a necessary and sufficient condition for the family to be able to approximate \emph{every} continuous invariant function, and hence for the corresponding architecture to be a universal approximator of invariant functions \citep{chen2019equivalence,dym2025low}.

The completeness of invariant architectures for point clouds has consequently received a great deal of attention. For example, the DeepSets model family  was proven to be complete on set-structured data \citep{deepsets}, while message-passing neural networks for graphs are not \citep{morris_review}. For point clouds, it is known that distance-based message-passing approaches are \emph{not} complete \citep{Pozdnyakov_2022}, while high-order invariant graph neural networks can attain completeness \citep{DBLP:conf/aaai/HordanAGD24, DBLP:conf/nips/RoseKRPB23, DBLP:conf/nips/LiWHZ23}. Yet, these approaches suffer from high computational cost, making them difficult to scale.

A popular methodology to increase the expressive power of geometric message passing neural networks while maintaining tractable runtime is to use higher order tensor representations of the rotation group. This approach is successfully used in many modern day models for point clouds such as MACE \citep{mace} and Equiformer \citep{equiformer}. In \cite{dymMaron} it is shown that  these higher order approaches are complete. The mechanism behind these results is the following: an arbitrarily high-order representation can express polynomials of an arbitrarily high degree, and the class of invariant polynomials of sufficiently high degree is complete. What has been missing, up to this work, is any quantitative control on \emph{how} high the representation order needs to be for completeness to hold.

In this work we provide the first answer to this question, in the form of explicit degree bounds for separating polynomial invariants of the relevant group action, that is the minimal polynomial degree needed to separate all orbits (defined precisely in \Cref{sec:lwr-bnd}). For point clouds of $n$ points in $\RR^d$, we show that the separation degree is lower bounded by a bound proportional to $n^2$ (\Cref{thm:lower-quadratic}), and upper bounded by a bound proportional to $dn^2$ (\Cref{thm:dn2}). Thus, in the main case of interest where $d=3$, this bound is tight, up to a constant. As a corollary, we deduce a quadratic lower bound and a cubic upper bound for the polynomial degree sufficient for separating weighted graphs  (\Cref{prop:graphs}).

We then use our bounds to derive  concrete completeness guarantees for existing invariant architectures: as we show in \Cref{sec:models}, it upgrades classical results in the geometric deep learning literature, which established completeness only for \emph{some} sufficiently high, unspecified, representation order, into an explicit quadratic bound on that order for models such as Tensor Field Networks \citep{TFN} and MACE \citep{mace}. Moreover, we show that, under locality assumptions on the point cloud, our results can be leveraged to obtain completeness for MACE with a representation order that depends quadratically only on the size of local neighborhood rather than on the size of the entire point cloud. We then show, by a counterexample construction, that if the representation order and correlation degree aren't sufficiently high then MACE cannot separate locally regular configurations (\Cref{prop:localLB}).

In summary, the main results of this paper are:
\begin{enumerate}
\item Tight quadratic degree bounds for separation of 3D point clouds by polynomials.
\item First quantitative completeness results for MACE and TFN, with quadratic dependence on point cloud cardinality.
\item Under locality assumptions on point clouds, an improvement of this bound to quadratic in the neighborhood size, and a lower bound showing that, with bounded depth and correlation order, the irreducible order must grow at least linearly in the neighborhood size.

\end{enumerate}

\subsection{Related Work}

\paragraph{Invariant Theory} 
The origins of Invariant Theory, which is the study of invariant polynomials, go back to the $19^{\rm th}$ century. Its goal is to understand invariant polynomials, including explicit description of the generators of the invariant ring, or bounds on the degree of these generators. A celebrated result by \cite{hilbert1890ueber} states that the ring of invariant polynomials is generated by a finite number of polynomials. Classical results in invariant theory focus on explicit derivation of these generators for specific group actions, or on bounding the cardinality or maximal degree of these generators. More recently there has been growing interest in the notion of separating invariants, which for compact groups are identical to the notion of complete invariants defined here. Generators are also separating invariants, and as a result the degree needed for separation cannot be larger than the maximal degree of the generators. 

To the best of our knowledge, there are no known results regarding the separation degree for point clouds. For weighted graphs, Thi\'ery \citeyearpar{thiery2000algebraic} proved an $O(n^4) $ upper bound on the degree of the generators. A bound with similar asymptotics can also be obtained from a general bound on permutation groups \citep{gobel1995computing}. In this paper we discuss the separation degree of weighted graphs, and prove a quadratic lower bound and a cubic upper bound. 

For more background on invariant theory and further related work see Appendix \ref{app:invariant_theory} and Appendix \ref{app:rl-wrk}. 

\subsection{Notation}\label{sub:notation}

We consider the setting of a group $G$ acting on a set $V$. We denote the orbit of an element $v\in V$ by $[v]:=\{gv| g\in G\}$ and say that $q:V\to \RR$ is $G$ invariant if $q(u)=q(v)$ for all $v\in V, u\in[v]$. Let $\mathcal{Q} $ be a family of $G$ invariant functions from $V$ to $\RR$. We say that $\mathcal{Q}$ is $(G,V) $ \emph{complete}, if for all $v,u\in V $ with $[v]\neq [u]$, there exists a function $q\in \mathcal{Q}$ such that $q(v)\neq q(u)$. 

Now assume additionally that $V$ is a finite dimensional real vector space, so that the notion of polynomials is defined. We denote  the set of  all polynomials $p:V\to \RR$ with real coefficients by $\RR[V]$, and the subset of polynomials of degree $\leq D$ by $\RR[V]_{\leq D}$.  The subset of $G$ invariant polynomials (of degree $\leq D$) are denoted by $\RR[V]^G$ (and $\RR[V]_{\leq D}^G$ ). The \emph{separation degree} of $(G,V) $, denoted by  $D(G,V)$, is  the smallest nonnegative integer $D$ such that $\RR[V]_{\leq D}^G$ is $(G,V) $ complete.

Our interest in this paper will be in completeness and separation degree for point clouds. A \emph{point cloud} with $n$ points in $\R^d$ is a tuple $X=(x_1,\dots,x_n)$ residing in the vector space $V=\R^{d\times n}$. We are interested in the action of rigid motions and permutations on this space. Explicitly, for a vector $t\in \RR^d$, an orthogonal matrix $R\in O(d) $, and a permutation $\sigma\in S_n $, this action is defined by 
\[
  (t,R,\sigma)\cdot X := \big(R(x_{\sigma^{-1}(1)}-t),\,\dots,\,R(x_{\sigma^{-1}(n)}-t)\big).
\]
We are  interested in four variants of this problem, which come up in different applications: we will consider groups of the form $G=H\times S_n$ where $H=E(d) $ (all rigid motions), $H=O(d)$ (the translation component $t$ is set to zero), $H=SO(d)$ (no translations, only proper rotations), and $H=SE(d) $ (with translations, only proper rotations). For each of these groups, we will derive bounds for the separation degree $D(G,\RR^{d\times n}) $.   

We will also be interested in the related notion of separation for \emph{weighted graphs}. A weighted graph is a symmetric $n \times n$ matrix with real coefficients. We denote the set of such matrices by  $\Symn $. 
The permutation group $S_n$ acts on $\Symn$ by conjugation. Namely, if $\tau\in S_n$ and $A\in \Symn $ then $\tau \cdot A$ is the matrix defined by 
$$\left(\tau\cdot A \right)_{ij}=A_{\tau^{-1}(i),\tau^{-1}(j)} .$$
If $A$ is a (weighted) adjacency matrix of a graph with some fixed node ordering, then $\tau \cdot A $ is the adjacency matrix of the same graph, with its nodes reordered by $\tau$. We use $1_n$ (or $0_n$) to denote the $n$-dimensional vector whose entries are all $1$ (or $0$).

\section{Degree  Bounds for Separating Polynomials on Point Clouds}
\label{sec:lwr-bnd}
In this section we answer the following questions: what is a polynomial degree required for separation and what polynomial degree is sufficient for separation of point clouds? We derive a lower bound for the polynomial degree of separation for point clouds, which is proportional to $n^2$ (Subsection \ref{subsec:sep-low-bnd}), and an upper bound proportional to $dn^2$  (Subsection \ref{subsec:sep-upper-bnd}). From this we derive a quadratic lower bound and a cubic upper bound on the separation degree for weighted graphs (\Cref{subsec:graphs}). 

\subsection{Lower Bound for Point Clouds}\label{subsec:sep-low-bnd}
We begin by stating our lower bound, that is a polynomial degree required for separating point clouds.

\begin{theorem}\label{thm:lower-quadratic}
  Let $d\geq 2, n\geq 3$. Consider the action of $G=H\times S_n$ on $V=\RR^{d\times n}$,  where $H\in \{E(d),SE(d),O(d),SO(d)\}$. Then the separation degree is lower bounded by 
    $$D(G,V)\geq n^2/4-1.$$
\end{theorem}
The remainder of this subsection is devoted to proving this theorem. We begin with some preliminaries. Firstly, the following lemma establishes a simple relationship between the separation degree with and without translations:
\begin{restatable}{lemma}{center}\label{lem:closing_options}
Let $V=\RR^{d\times n} $ and $\Vcent$ denote the subspace of point clouds $X\in \RR^{d \times n}$ with $X1_n=0_d$. Then
\begin{align}
D(E(d)\times S_n,V)&=D(O(d)\times S_n,\Vcent)\leq D(O(d)\times S_n,V) \label{eq:equiv1}\\
D(SE(d)\times S_n,V)&=D(SO(d)\times S_n,\Vcent)\leq D(SO(d)\times S_n,V) \label{eq:equiv2}
\end{align}
\end{restatable}
\begin{proof}[Proof idea]
Equation \ref{eq:equiv1} follows from the fact that two point clouds $X,Y$ are in the same $E(d)\times S_n $ orbit, if and only if the centralized point clouds $X-\frac{1}{n}X1_n1_n^T, Y-\frac{1}{n}Y1_n1_n^T$ are in the same $O(d)\times S_n $ orbit. The proof for \eqref{eq:equiv2} is identical. 
\end{proof}

The proof of the theorem is strongly based on the following proposition:
\begin{proposition}\label{prop:lowerDbound}
Let $G$ be a compact group acting orthogonally on a finite dimensional real inner product space $V$. 
    Let $k\geq 1$ be an integer. Suppose there exists a $2$-dimensional subspace $W\subseteq V$ and an element $g\in G$ such that $g\cdot W\subseteq W$ and $g$ acts by rotation over $2\pi/k$ on $W$. If the unit circle $S(W)\subset W$ contains more than $1$ orbit, then 
     we have $D(G,V)\geq k$.
\end{proposition}
\textit{Proof.}
We choose an orthonormal basis of $W$
and identify $\R[W]\cong \R[x_1,x_2]$. 
Restricting polynomials gives a surjective ring homomorphism $\R[V]\to \R[x_1,x_2]$. Let $H$ be the subgroup of $G$ generated by $g$. 
It is well-known that the invariant ring $\R[x_1,x_2]^H$  is generated by $x_1^2+x_2^2$ and the polynomials $f_1,f_2\in \R[x_1,x_2]^H$ of degree $k$ given by 
$(x_1+i x_2)^k=f_1+if_2$ (see for example~\citealp{Yasinsky2022}).
In particular, every element of $\R[x_1,x_2]^H$ of degree $<k$ is a 
polynomial in $x_1^2+x_2^2$ and is constant on $S(W)$. Since not all elements in $S(W)$ lie in the same $G$-orbit, there exists an invariant $f\in \R[V]^G$ of degree $\leq D(G,V)$ that is not constant on $S(W)$. If $f'$ is the restriction of $f$ to $W$, then $f'\in \R[W]^H=\R[x_1,x_2]^{H}$ is not constant on $S(W)$, so we have $D(G,V)\geq \deg(f)\geq \deg(f')\geq k$. 
\qed{}

We now return to the proof of Theorem \ref{thm:lower-quadratic}. By Lemma \ref{lem:closing_options} it is sufficient to prove a bound for all four groups in the lemma, by considering the two groups $G=O(d)\times S_n$ or $G=SO(d)\times S_n$, and the vector space $\Vcent $- the space of centralized point clouds. Throughout, we will assume $G$ is one of these two groups. 

We first consider the case $d=2$.  We identify $\R^2$ with $\C$ and $\R^{2\times n}$ with $\C^n$. 
If $n=2m$ is even, we define $p=m+1$, $q=m-1$, $u=1$ and $r=m$. If $n=2m+1$ then we define $p=m+1$, $q=m$, $u=2$ and $r=2m+1$. We have $n=p+q$, $r=up-1=uq+1$ and 
$pq\geq n^2/4-1$.
For $t\in \R$ and all $j\in \Z$, define 
$$
a_j(t)=e^{i(2\pi j/p+t)}\quad\mbox{and}\quad
b_j(t)=e^{i(\pi/2-2\pi j/q-t)}.
$$
So $a_1(t),a_2(t),\dots,a_{p}(t)$ are the counterclockwise consecutive vertices of a regular $p$-gon centered around the origin. The vertices $b_1(t),b_2(t),\dots,b_{q}(t)$ are clockwise consecutive vertices of a  regular $q$-gon.
As $t$ increases, the $p$-gon rotates counter-clockwise and the $q$-gon rotates clockwise.
We easily verify that
$a_{j}(t+\frac{\pi}{2})=ia_j(t)$, $b_j(t+\frac{\pi}{2})=-ib_j(t)$, $a_{j+p}(t)=a_j(t)$, $b_{j+q}(t)=b_j(t)$ and
\begin{eqnarray*}
    a_j(s+t) & = & e^{is}a_j(t)=(\cos(s)+i\sin(s))a_j(t)=\textstyle \cos(s)a_j(t)+\sin(s)a_j(t+\frac{\pi}{2})\\
       b_j(s+t) & = & 
    e^{-is}b_j(t)=(\cos(s)-i\sin(s))b_j(t)=\textstyle \cos(s)b_j(t)+
    \sin(s)b_j(t+\frac{\pi}{2})
\end{eqnarray*}
 Define  
$$
w(t)=\begin{pmatrix}
    a_1(t) & a_2(t) & \cdots & a_p(t) & b_1(t) & b_2(t) & \cdots & b_q(t)
\end{pmatrix}.
$$
The coordinates of $w(t)$ form a point cloud in the complex plane (see figure~\ref{fig:1}) and this point cloud is centralized.
\begin{figure}%
    \includegraphics[width=0.33\textwidth]{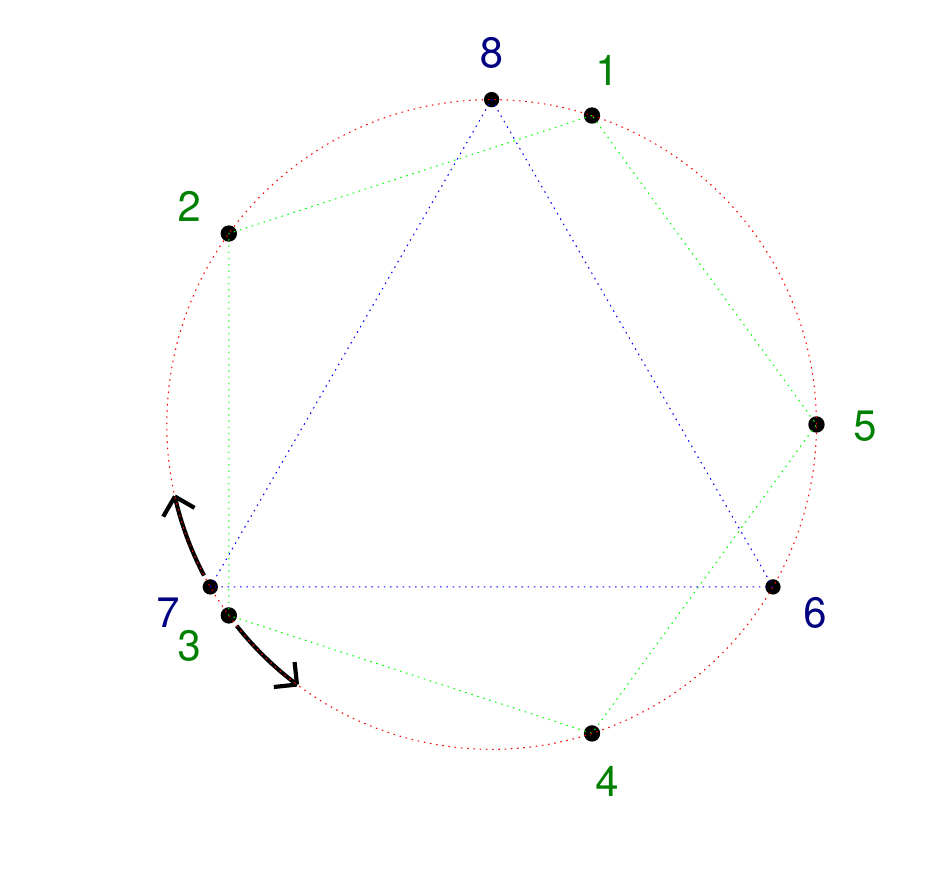}
        \includegraphics[width=0.33\textwidth]{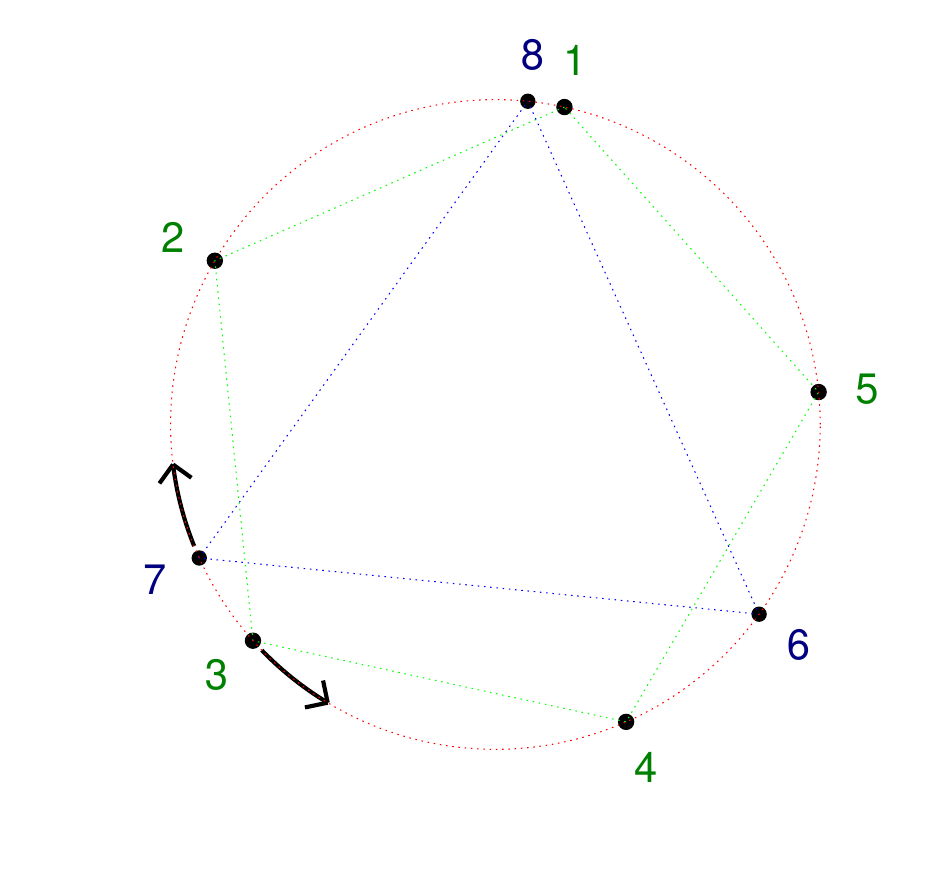}
          \includegraphics[width=0.33\textwidth]{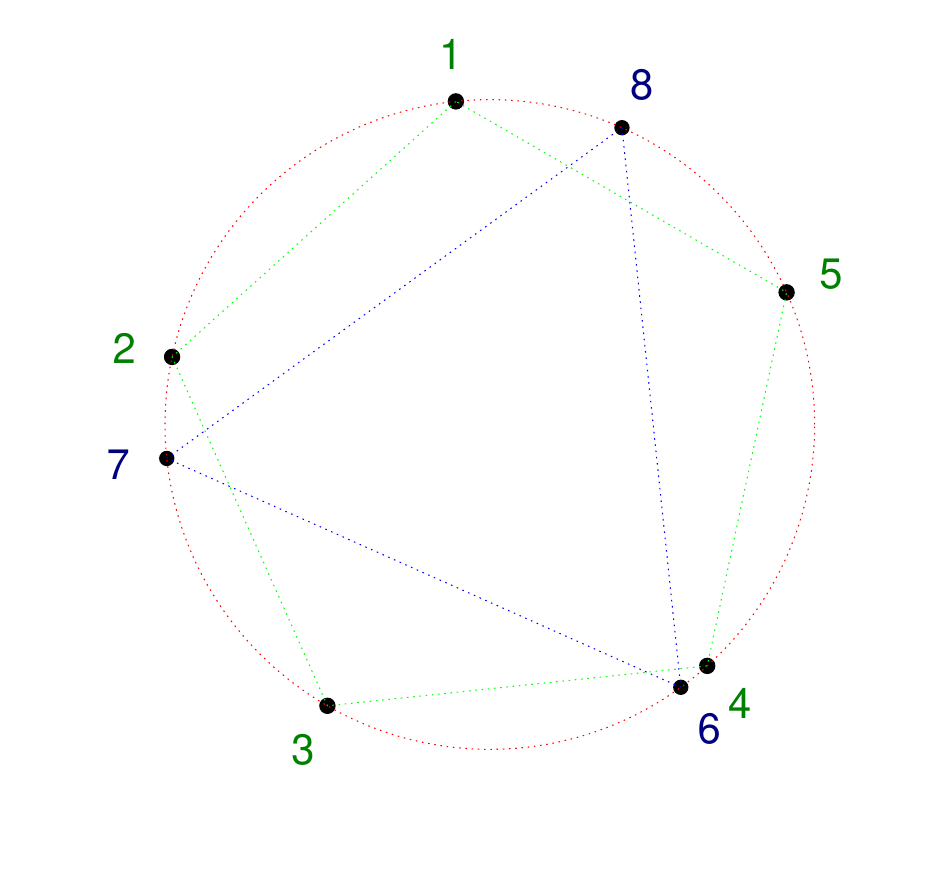}
    \caption{From left to right we have plotted the  coordinates of $w(0)$, $w(0.1)$ and $w(\frac{2\pi}{15})=w(0.4189)$ in the complex plane for $n=8$. We have $p=5$, $q=3$ and $r=4$. As $t$ goes from $0$ to $0.1$ the pentagon  moves counterclockwise and the triangle  moves clockwise. The point clouds of $w(0)$ and $w(\frac{2\pi}{15})$ are isometric.
    The latter is obtained from the first by rotation over $\frac{8\pi}{15}$ and applying the permutation $(1\ 2\ 3\ 4\ 5)(8\ 7\ 6)$ to the columns of $w$.}
    \label{fig:1}
\end{figure}
We have $w(t+s)=\cos(s)w(t)+\sin(s)w(t+\frac{\pi}{2})$. This shows that $w(t)$ lies in the real span of the two linearly independent matrices $w(0)$ and $w(\frac{\pi}{2})$. 
 Now we take $s=\frac{2\pi}{pq}$.
Because $r=up-1=uq+1$, we have  
\begin{eqnarray*}
\textstyle a_i\big(t+\frac{2\pi}{pq}\big) &=& \textstyle a_i\big(t+\frac{2\pi r}{pq}-\frac{2\pi u}{p}\big)=a_{i-u}\big(t+\frac{2\pi r}{pq}\big)=e^{2\pi i r/pq}a_{i-u}(t)\\
\textstyle b_j\big(t+\frac{2\pi }{pq}\big) &=& \textstyle b_j\big(t-\frac{2\pi r}{pq}+\frac{2\pi u}{q}\big)=b_{j+u}\big(t-\frac{2\pi r}{pq}\big)=e^{2\pi i r/pq}b_{j+u}(t)
\end{eqnarray*}
and
$$
w\big(t+{\textstyle \frac{2\pi}{pq}}\big)=e^{2\pi i r/pq}
\begin{pmatrix}
    a_{1-u}(t) & a_{2-u}(t) & \cdots & a_{p-u}(t) & b_{1+u}(t) & b_{2+u}(t) & \cdots & b_{q+u}(t)
    \end{pmatrix}
    $$
The action of $\SO(2)$ on $\R^{2\times n}\cong \C^n$ is just multiplication with unit scalars in $\C$. The permutation
$$\sigma=(1\ 2\ \dots\ p)^{u}(p+1\ p+2\ \cdots\ p+q)^{-u}\in S_n
$$
acts on $\C^n$ by permuting the entries (or permuting the columns in $\R^{2\times n}$).
For $g=(e^{2\pi i r/pq}, \sigma)\in \SO(2)\times S_n\subseteq G$ we have $g\cdot w(t)=w(t+\frac{2\pi}{pq})$. Recall that $W$ is the span of $w(0)$ and $w(\frac{\pi}{2})$. From this it is clear that $g\cdot W\subseteq W$. If we restrict the action of $g\in G$ to the subspace $W$, then it is just rotation over the angle $2\pi/pq$.

 If we forget about the labeling of the vertices, then $w(t)$ is a point cloud consisting of the union of the vertices of a $p$-gon and the vertices of a $q$-gon.
Now $w(t)$ is obtained from $w(0)$ by rotating the $p$-gon counterclockwise and the $q$-gon clockwise. From this it is clear that
for small $t>0$, $w(0)$ and $w(t)$ are not isometric point clouds.
This proves that not all the elements of $S(W)\subset W$ lie in the same $\OO(2)\times S_n$-orbit.
By Proposition~\ref{prop:lowerDbound} we get $D(G,V)\geq pq\geq n^2/4-1$.

Now assume that $d>2$. We view $\R^{2\times n}$
as the subspace of $\R^{d\times n}$ of $d\times n$ matrices whose rows $3,\dots,d$ are all equal to $0$, and let $W\subseteq \R^{2\times n}$ be the $2$-dimensional subspace that was constructed in the $d=2$ case. Now there is again an element $g\in \SO(2)\times S_n\subseteq G\subseteq\OO(d)\times S_n$ that acts on $W$ by rotation over the angle $2\pi/pq$. For similar reasons as before, not all elements of $S(W)\subset W$ lie in the same $\OO(d)\times S_n$-orbit.
So we again can apply Proposition~\ref{prop:lowerDbound} to conclude that $D(G,V)\geq pq\geq n^2/4-1$.

\qed{}

\subsection{Upper bound for point clouds}\label{subsec:sep-upper-bnd}
Having established a lower bound on the degree required for separation of point clouds, we proceed to provide an upper bound. In the proof of \Cref{thm:dn2}, we introduce a polynomial construction with degree that depends linearly on the dimension $d$ and quadratically on the number of points that separates $d$-dimensional point clouds, thereby establishing an upper bound on the degree sufficient for separation. 

\begin{theorem}\label{thm:dn2}
    Let $n\geq 2$ and $1\leq d\leq n$. Then
    \begin{align*}
      D\bigl(H\times S_n,\Rdnmat\bigr)\;&\leq\;6dn^2-4dn-d(d-1)(n-1), \quad \quad \, \, \, \, \, \text{ for } H\in \{\OO(d),E(d) \}\\
      D\bigl(H\times S_n,\Rdnmat\bigr)\;&\leq\;6dn^2-4dn-d(d-1)(n-1)+d, \quad  \text{ for } H\in \{\SO(d),SE(d) \}
    \end{align*}
\end{theorem}

\begin{proof}[Proof idea]
By Lemma \ref{lem:closing_options} it is sufficient to prove the claim for $O(d)\times S_n$ and $SO(d)\times S_n $. The goal is to find, for any given $X,Y\in \RR^{d\times n}$ from different orbits, an invariant polynomial of the advertised degree that separates them. Since two clouds are $\OO(d)\times S_n$-equivalent iff their Gram matrices $A=X^TX$, $B=Y^TY$ are equal up to simultaneous row/column permutation, it suffices to separate $A$ and $B$ by an $S_n$-invariant polynomial $Q$ (at the cost of doubling the degree). We choose a polynomial of the form $$Q(C)=\sum_{\tau\in S_n}\chi(\tau\cdot C)$$
where $\chi$ is a polynomial which takes values in $\{0,1\}$ on the $S_n$-orbits of $A$ and $B$, with $\chi(A)=1$ and $\chi(\tau\cdot B)=0$ for all permutations $\tau$. The task is thus reduced to realizing such an indicator $\chi$ by a low-degree polynomial. Clouds of different rank are separated by $\mathrm{tr}\big((X^TX)^k\big)$ for some $k\leq n$, so we may assume $\mathrm{rank}(A)=\mathrm{rank}(B)=r\leq d$. Then $A$ can be relabeled so that its leading $r\times r$ block is nonsingular, and among matrices of rank $r$ it is determined by its first $r$ rows, so at most $d\cdot n$ entries must be checked. Each entry is checked by a Lagrange polynomial of degree $<n$ interpolating on the values of the corresponding row of $A$, which is valid thanks to a row indicator of degree $O(n^2)$ (via power sums), which ensures the row of $\tau\cdot C$ is a rearrangement of that of $A$. The total degree is $O(dn^2)$. For $\SO(d)$, if the point clouds have rank $d$, we multiply each summand by an affine function of the determinant of the first $d$ points, which detects orientation and adds $d$ to the degree; if their rank is smaller than $d$, the $\SO(d)\times S_n$-orbit of $X$ is equal to its $\OO(d)\times S_n$-orbit, and the polynomial for $\OO(d)$ separates them.
\end{proof}

\Cref{thm:dn2} shows that it is possible to obtain separating invariants of degree quadratic in $n$ for point clouds. It follows that neural networks that are able to express polynomials up to that degree bound can provably separate point clouds with $n$ points. Many popular machine learning equivariant models are able to express high order polynomials using tensor representations, allowing us to show how this upper bound proves the separation capabilities of prominent such models in \Cref{sec:models}.

\subsection{Degree bounds for Weighted graphs}\label{subsec:graphs}
Our degree bounds for point clouds automatically imply degree bounds for weighted graphs (see definition in Subsection \ref{sub:notation}). 
 
\begin{restatable}{proposition}{graphs}\label{prop:graphs}
For any natural $n$, we have that
$D(O(n)\times S_n,\RR^{n\times n})=2D(S_n,\Symn)$.
    In particular it follows from Theorem \ref{thm:lower-quadratic} and Theorem \ref{thm:dn2}, that for $n\geq 3$,
    $$\frac{1}{2}(n^2/4-1) \leq D(S_n,\Symn)\leq \frac{1}{2}\left(5n^3-2n^2-n \right)$$
  
\end{restatable}
\begin{proof}[Proof idea]
The proof that $D(O(n)\times S_n,\RR^{n\times n})=2D(S_n,\Symn)$ is based on the observation that two positive semi-definite weighted graphs  $A,B$ are in the same $S_n$ orbit, if and only if there are $O(n)\times S_n $ equivalent point clouds $X,Y$ satisfying $X^TX=A, Y^TY=B $. The full proof is in the appendix.
\end{proof}

\section{ Quantitative approximation bounds for Invariant Models}\label{sec:models}

Our upper bound in Theorem \ref{thm:dn2} shows universality and separation for $\Gdn $ invariant polynomials of degree proportional to  $d\cdot n^2 $. These results can be extended to similar quantitative universality results for $\Gdn $ invariant architectures which are based on message passing and high dimensional representations of the rotation group, since such architectures can construct invariant polynomials. We give two specific architectures for which such results can be obtained: Tensor Field Networks (TFN) and MACE.

In this section we will often specialize to the case $d=3$, since this is the most relevant dimension in practice, and as a result many invariant models, including TFN and MACE, are defined for $d=3$ only.

\subsection{Tensor Field Networks}
Tensor Field Networks (TFN) \citep{TFN} pioneered the approach of using high-dimensional representations of the rotation group for $\Gthree$ invariant models. We give a brief background on this approach.

A real representation of $SO(3) $ is a real vector space on which $SO(3)$ acts linearly. A standard example is the action of $SO(3)$ on $\RR^3$ by matrix multiplication. More generally, for every $\ell=0,1,2,\ldots,$ there is a $2\ell+1$ dimensional representation $W_\ell\cong \RR^{2\ell+1} $ of $SO(3)$.  This representation is obtained via spherical harmonics: For a fixed $\ell$, the $\ell$-th spherical harmonics are a collection of $2\ell+1 $ polynomials $Y_\ell^m:\RR^3\to \RR, m=-\ell,\ldots,\ell $ which are (a) $\ell$-homogeneous and (b) are harmonic and satisfy that the mapping $x\mapsto \left( Y_\ell^m(x), m=-\ell,\ldots, \ell \right) $
is a rotation equivariant function mapping the standard representation $\RR^3\cong W_1$ into a high dimensional representation $W_\ell $.

% An important property of these polynomials is that for every fixed rotation $R$, the polynomial $Y_\ell^m(Rx) $ is a linear combination of the other spherical harmonics, namely 
% $$Y_\ell^m(Rx)=\sum_{m'=-\ell}^\ell D^{(\ell)}_{mm'}(R) Y_\ell^{m'}(x)$$

To what extent does the employment of higher dimensional features  increase the expressive power of the model? In \cite{Pozdnyakov_2022} it was shown that methods which use geometric message passing on the distances only (corresponding to $\ell=0$ features) are not complete. In contrast, in \cite{dymMaron} it was shown that TFN can construct all invariant polynomials up to degree $D$, if the feature order and depth are proportional to $D$. Specifically, the following is a simple corollary of Theorem 2 in that paper:
\begin{proposition}\label{prop:TFN}
Let $D\in \NN$. Any $\Gthree$ invariant polynomial $p:\RR^{3\times n}\to \RR $ can be expressed by a TFN model with $2D$ layers, and irreducible features of maximal order $\ellmax=D$.  
\end{proposition}
In \cite{dymMaron}, this proposition was used to prove that one can always obtain a complete model if the maximal irreducible order is taken to be arbitrarily large, where this result is based on the observation that invariant polynomials can approximate all continuous invariant functions. With our result, Theorem \ref{thm:dn2}, we can now for the first time obtain a bound on the order of $\ellmax$ needed:
\begin{corollary}\label{cor:tfn}
Let $n\in \NN$ and set $D=18n^2-18n+9$. Then an invariant TFN model with $2D$ layers, and irreducible features of order $\leq D$, is $(\Gthree,\RR^{3\times n})$ complete.
\end{corollary}

The derivation of Proposition \ref{prop:TFN} from the results in \cite{dymMaron}, and some more details of the TFN model used, are discussed in Appendix \ref{app:tfn}.

\subsection{MACE}
MACE \citep{mace} is a very successful $\Gthree$ invariant model. Its main innovation is the employment of tensor products of several high order features. Namely, if $A_1,\ldots,A_\nu $ are features in $W_{\ell_1},\ldots,W_{\ell_\nu} $ respectively, then one can obtain a new  feature $A_1\otimes A_2\otimes \ldots \otimes A_\nu $ in the tensor product representation $W_{\ell_1}\otimes W_{\ell_2}\otimes \ldots\otimes W_{\ell_\nu} $. Moreover, using the Clebsch-Gordan coefficients, this tensor product representation can be mapped equivariantly to a direct sum of $W_\ell$ representations. While previous works like TFN employed such tensor products between pairs of tensors ($\nu=2$), MACE was the first model to use such decompositions for $\nu>2$. Our next proposition gives an explanation for the effectiveness of this procedure, by showing that unlike TFN, for which the construction of \cite{dymMaron} uses depth $2D$ to express invariant polynomials of degree $D$, MACE can do this with depth=1, as long as the maximal number of tensor products (known as the body order) $\numax$, and the maximal representation order $\ellmax$, satisfy 
$\numax=\min\{D,n-1\}$ and $\ellmax= \lfloor D/2 \rfloor$, respectively.

We modify the step of MACE that forms tensor products of the aggregated features so that it follows the atomic cluster expansion (ACE) \citep{ace}: each factor of the tensor product has its own channel mixing, as each factor of an ACE basis function has its own radial index. This modification, and the other differences from the original model, are detailed in Appendix~\ref{app:MACE}. 

\begin{restatable}{proposition}{propmace}
\label{prop:mace}
MACE with $\numax=\min\{D,n-1\}$ and $\ellmax = \lfloor D/2 \rfloor$ can express any $\Gthree $ invariant polynomial of degree $\leq D $. 
\end{restatable}

As a corollary of this proposition and Theorem \ref{thm:dn2}, we obtain the following completeness result for MACE
\begin{corollary}\label{cor:maceLocal}
Let $n\in \NN$ and set $D=18n^2-18n+9 $. Then a MACE model with depth 1, $\numax=n-1$ and $\ellmax = \lfloor D/2 \rfloor$ is $(\Gthree,\RR^{3\times n})$ complete.
\end{corollary}

The details of the MACE model used for Proposition \ref{prop:mace}, and the proof of the proposition, are given in Appendix \ref{app:MACE}. In Appendix \ref{app:MACE_multi} we show that when message passing is on all points, more iterations can be used instead of a high body order, for example $\numax=2$ with $\log_2(n+1)-1$ iterations (Corollary \ref{cor:mace_multi}).    

\subsection{Improved MACE Completeness for locally regular point clouds}\label{lg:sec}
The completeness results of Section~\ref{sec:models} require invariants of degree $18n^2-18n+9$, a quantity that grows with the number of points in the cloud. This degree is likely too high for most practical MACE implementations. Moreover, MACE typically applies message passing only within a fixed cutoff radius \citep{mace}. Accordingly, the relevant scale is not $n$ but the maximal number $K$ of points inside a
single cutoff ball. In this subsection we isolate a set of properties of a point cloud which
make this precise: on the class of clouds satisfying them, a \emph{local} model of MACE type is
complete using irreducible features and invariants of degree at most $6dK^2$, a bound
independent of $n$ and of the depth of the model. Restricted to such a class, this strengthens the MACE completeness
corollary of Section~\ref{sec:models}. Conversely, on a similar class, we prove a lower bound showing that MACE cannot separate without irreducibles and correlation orders whose total degree is $\sim K^2 $. 

Below we will fix  a cutoff radius $\rho>0$ and an integer $K$. 
In order to focus on applications to MACE, we will set the dimension to $d=3$. Proofs are given in Appendix~\ref{app:local} and they are written for general $d\ge2$, except where the constants are specific to $d=3$.

\paragraph{Environments.}
Let $X=(x_1,\dots,x_n)\in\R^{d\times n}$ and write $[n]:=\{1,\dots,n\}$. For $i\in[n]$ let
\[
\mathcal B^X_1(i):=\{\,j\in[n]:\ \|x_j-x_i\|\le\rho\,\}
\]
be the set of indices in the cutoff ball at $x_i$. Two indices $i,j\in[n]$ are
\emph{$\rho$-adjacent} if $0<\|x_i-x_j\|\le\rho$ and the columns indexed by $\mathcal B^X_1(i)\cap\mathcal B^X_1(j)$ contain $d+1$ affinely independent points.

\begin{definition}[locally regular point clouds]\label{lg:def:reg} 
A matrix $X=(x_1,\dots,x_n)\in\R^{d\times n}$ with $n\ge2$ is \emph{$(\rho,K)$-regular} if
\begin{description}
\item[\textnormal{(D)}] the columns $x_1,\dots,x_n$ are pairwise distinct;
\item[\textnormal{(K)}] $\bigl|\{\,j\in[n]:\|x_j-y\|\le\rho\,\}\bigr|\le K$ for every $y\in\R^d$;
\item[\textnormal{(O)}] the $\rho$-adjacency graph on $[n]$ is connected.
\end{description}
We say that $X \in \RR^{d\times n}$ is $(\rho,K)$-\textbf{ strongly regular}, if $X$ is $(\rho,K)$ regular, and additionally, $X$ does not contain two different ordered subsets of $d+1$ points which are congruent. Here two ordered subsets $(p_0,\ldots,p_d)$ and $(p_0',\ldots,p_d')$ are congruent if there exist $g\in O(d)$ and $b\in \RR^d$ such that $p_s'=gp_s+b$ for all $s=0,\ldots,d$. We denote the collection of $(\rho,K)$ regular and $(\rho,K)$ strongly regular  point clouds by $\mathcal R_{\rho,K}$ and $\mathcal R^S_{\rho,K} $, respectively. 
\end{definition}

Hypothesis \textnormal{(K)} bounds the local multiplicity from above and is what keeps
environments small. Hypothesis \textnormal{(O)} bounds it from below and is what makes the
environments rigid: connectivity lets information propagate through the cloud, while the overlap
requirement built into $\rho$-adjacency forces the cutoff balls of adjacent points to meet in a
set of full affine dimension.

Our main result of this section states that, for $(\rho,K)$ strongly regular point clouds, a
single iteration of local MACE with irreducible order $\sim K^2$ is already complete, and that
further iterations preserve completeness:
\begin{restatable}{theorem}{maceLocal}\label{thm:maceLocal}
Let $K\leq n$ be natural numbers and $\rho>0$, and let $G=\Gthree$. Set
$$D(j)=18j^2-18j+9,\qquad D=D(K). $$
Then local MACE with cutoff $\rho$, $T\geq 1$ iterations, $\numax=K-1$ and
$\ellmax=\lfloor D/2 \rfloor$, is  $(G,\mathcal R^S_{\rho,K}) $ complete.
\end{restatable}

\begin{proof}[Proof Idea]
The main idea is that, after a single MACE iteration, one obtains a set of node features $\{E_i(X), \, i=1\ldots,n\} $. These features are invariants, so reconstructing an environment from one of them is a separation problem for point clouds of $K$ points, and this is where Theorem \ref{thm:dn2} enters: it bounds the degree at which invariants separate such clouds by $D=D(K)$, and Corollary \ref{cor:maceLocal}, which is that theorem combined with the expressivity of one MACE layer, turns the bound into the degree assumption made here. Accordingly $\mathcal B^X_1(i)$, the local environment of each node $i$, can be reconstructed from $E_i(X)$, up to a rotation and a translation. Strong regularity then implies that, whenever the two environments share $d+1$ affinely independent points (up to congruence), these points can be identified with each other; by the overlap requirement in the definition of $\rho$-adjacency this is the case for every adjacent pair $i,j$, so the environments of adjacent nodes can be glued. Since the $\rho$-adjacency graph is connected, and the intersections contain $d+1$ affinely independent points, the whole point cloud can be reconstructed uniquely. A full proof is given in the appendix.
\end{proof}

We accompany the above by the following lower bound. In this lower bound, the hyperparameters $\ellmax,\numax$ correspond to the maximal irreducible order and maximal correlations used in the MACE iterations, and $\nuread$ denotes the correlations used in the final readout stage.

\begin{proposition}[proved in Appendix~\ref{app:local:lower:rings}]\label{prop:localLB}
Let $K\geq7$, and let a local MACE model with cutoff $\rho$, message-passing depth $T$ and hyperparameters $\ellmax,\numax,\nuread$ be $(\Gthree,\mathcal R_{\rho,K})$ complete on clouds of $K$ points. Then
$$\nuread\cdot\ellmax\cdot\sum_{t=1}^{T}\numax^{\,t}\ \geq\ 2\Bigl\lfloor\frac{K-1}{2}\Bigr\rfloor ,$$
and if the readout is that of \eqref{eq:Ei}--\eqref{eq:E}, a function of the invariant channels at each point followed by a sum, then already $\ellmax\sum_{t=1}^{T}\numax^{\,t}\geq2\lfloor(K-1)/2\rfloor$.
\end{proposition}

\begin{proof}[Proof idea]
The counterexample is a local version of the counter-rotating polygons of Theorem~\ref{thm:lower-quadratic}: for $p,q\ge3$ we form a family of clouds of $K=p+q+1$ points in $\RR^3$, each consisting of a point at the origin and two rings of $p$ and $q$ points with free phases $\alpha,\beta$ around a common axis (see Figure~\ref{fig:rings}). The cutoff radius is such that every distance the model can read is constant along the family, the cross-ring distances that do vary all exceeding the cutoff, so only angular resolution is at stake, while the sole invariant $\alpha-\beta$ is defined modulo $2\pi/\mathrm{lcm}(p,q)$. Each spherical harmonic of degree $\ell$ in a message contributes at most $\ell$ to the frequency in $(\alpha,\beta)$ of the features, and products of $\numax$ messages add frequencies, which gives the budget on the left. Shifting $\beta$ by $2\pi/q$ leaves the cloud and each point of the first ring fixed, and a rotation about the axis shifts $\alpha$ and $\beta$ together, so the output at such a point is constant unless one of its Fourier modes $(m_1,m_2)$ has $|m_1|+|m_2|\geq 2q$; symmetrically, $2p$ for the second ring, and more for the center. A model whose point outputs are all constant separates nothing, and choosing $p$ and $q$ as equal as possible gives $2\min(p,q)=2\lfloor (K-1)/2\rfloor$.
\end{proof}
The significance of this proposition is that MACE with the modestly sized hyperparameters usually selected in practice is not complete even under the regularity assumptions. To illustrate this: for $\ellmax=\numax=3$ and the readout of \eqref{eq:Ei}, the family is not separated once $K\geq11$ for $T=1$, $K\geq39$ for $T=2$ and $K\geq119$ for $T=3$. Figure~\ref{fig:rings} shows a $12$-point configuration whose relative polygon rotations can be distinguished by MACE only if
\begin{wrapfigure}{r}{0.28\linewidth}
  \vspace{-1em}
  \centering
  \includegraphics[
    width=\linewidth,
    trim=60bp 45bp 60bp 55bp,
    clip
  ]{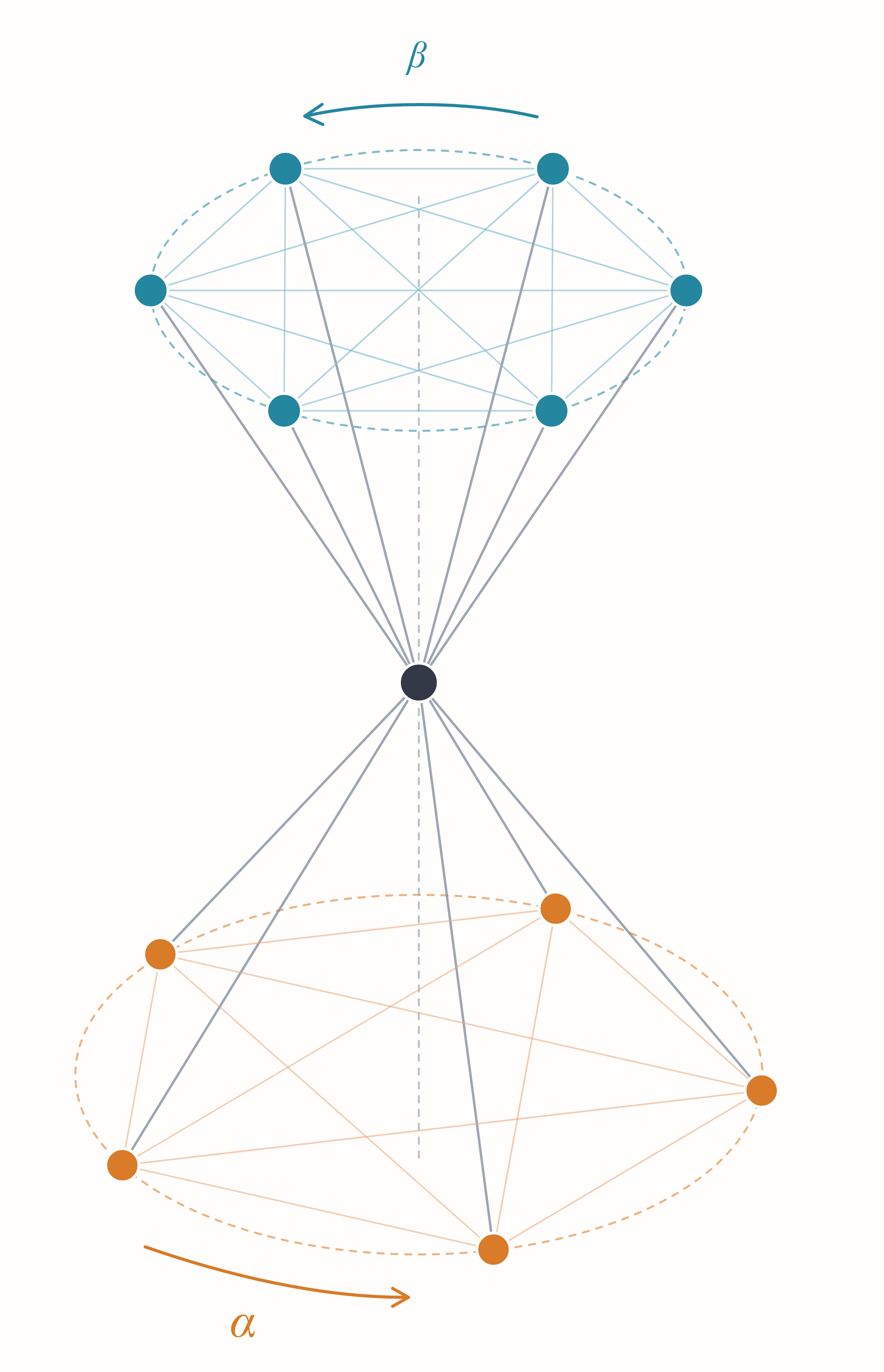}
  \caption{A family of regular point clouds which cannot be separated by mace without high degree/correlation order. See Proposition \ref{prop:localLB} for details.}
  \label{fig:rings}
  \vspace{-1em}
\end{wrapfigure}
\[
\nu_{\mathrm{read}}\,\ell_{\max}\sum_{i=1}^{T}\nu_{\max}^{\,i}\ \geq\ 2\Bigl\lfloor\frac{11}{2}\Bigr\rfloor=10 .
\]
\section{Conclusion, Limitations and Future Work}
In this paper we showed that the degree of separation for point clouds is lower bounded by $\sim n^2$ and upper bounded by $\sim dn^2$. We showed this implies a $\sim n^2$ lower bound and $\sim n^3$ upper bound for weighted graphs, and used the upper bound for point clouds to derive completeness results for TFN and MACE with irreducible order which grows quadratically in $n$, or under additional assumptions, quadratically in the maximal neighborhood size. We also provide a lower bound for the MACE architecture. To the best of our knowledge these are the first quantitative results of this kind.

In the context of polynomial separation, the remaining gap is that the upper bound  exceeds the lower bound by a factor of order $d$. This is essentially equivalent to closing the gap between the quadratic lower bound and cubic upper bound attained for weighted graphs. In the context of invariant networks there is a more substantial gap:  Our results used the fact that TFN and MACE can construct invariant polynomials up to a degree determined by the irreducibles used. However, these models can also express non-polynomial functions. Thus, we do not have tight lower bounds for these models. An interesting open question is to give tight upper and lower bounds on the degree of separation of these models. In particular, are MACE and TFN complete with an irreducible degree that does not depend on $n$? For regular point clouds Proposition~\ref{prop:localLB} answers it in part: with bounded depth and correlation order the irreducible degree must grow at least linearly in the neighborhood size. Whether linear growth suffices, the question for the class of all point clouds, and the case of MACE architectures with unrestricted neighborhoods, remain open. 

\subsection*{AI use statement}

We used generative AI tools to edit and improve the writing, check manuscript formatting, and assist in verifying the correctness of our results. We did not use generative AI tools for research ideation. All AI-assisted suggestions and checks were reviewed by the authors. The authors take full responsibility for the final content of this work, including all claims, proofs, and conclusions.
\subsection*{Acknowledgments}

S.H. and N.D. acknowledge funding by the Israel Science Foundation grant no. 272/23 and \mbox{DFG-ISF} grant no. 3597/26. S.H. is supported by the Gloria and Ken Levy Foundation Fellowship. H.D. was partially supported by a Simons Foundation International grant administered by the Simons Foundation (SFI-MPS-SFM-00010862). M.P. was supported by the Chilean ANID basal center CenIA, code FB210017.

\bibliography{iclr2027_conference}
\bibliographystyle{iclr2027_conference}
\newpage
\appendix

\section{Further Related Work}\label{app:rl-wrk}

\paragraph{Completeness and universality in Geometric Deep Learning}
Understanding the expressive power of machine learning models for point clouds allows researchers to have a benchmark for model comparison. \citet{dymMaron} characterized the polynomial expressivity of models that utilize high-dimensional tensors, often termed equivariant geometric machine learning models, and proved such models are universal at the limit of tensor order \citep[e.g.,][]{TFN}. Subsequent work used tools developed there to prove universality of novel architectures \citep{weitao_clofnet_2021, gasteiger_gemnet_2021, DBLP:conf/tagml/FinkelshteinBMD22}. 

Universality was also proven for sufficiently high-order geometric graph neural networks (GNNs), which process pairwise invariants and computes high-order correlations among them \citep{DBLP:conf/icml/HordanAD24,DBLP:conf/aaai/HordanAGD24, DBLP:conf/nips/RoseKRPB23, DBLP:conf/nips/LiWHZ23}. Yet equivariant models and GNNs are structured differently, with nonlinearities embedded throughout the components of GNNs, in contrast with equivariant models which are composed of linear maps except for the radial functions and readout aggregation, so those results do not carry over to equivariant GNNs. High order GNNs incur high computational cost and poor scalability compared with equivariant models \citep{mace}, leaving quantitative completeness in equivariant models an open question, up to this manuscript.

\section{Missing proofs}
\center*
\begin{proof}
Throughout, let $H_0\in\{O(d),SO(d)\}$, and let $H$ denote the corresponding group with translations, i.e.\ $H=E(d)$ if $H_0=O(d)$ and $H=SE(d)$ if $H_0=SO(d)$. We prove
\[
D(H\times S_n,V)=D(H_0\times S_n,\Vcent)\leq D(H_0\times S_n,V),
\]
which gives \eqref{eq:equiv1} and \eqref{eq:equiv2} simultaneously.

\emph{Matrix form of the action.} For $\sigma\in S_n$ let $P_\sigma$ be the permutation matrix with $P_\sigma e_j=e_{\sigma(j)}$. The $j$-th column of $XP_\sigma^T$ is $x_{\sigma^{-1}(j)}$, so the action defined in \Cref{sub:notation} reads
\[
(t,R,\sigma)\cdot X \;=\; R\,(X-t1_n^T)\,P_\sigma^T ,
\]
and $H_0\times S_n$ is identified with the subgroup $\{(0,R,\sigma)\}$ of $H\times S_n$. Note that $P_\sigma 1_n=1_n$ and hence also $1_n^TP_\sigma^T=1_n^T$.

\emph{The centering map.} Let $C:=I_n-\frac1n1_n1_n^T$ and define the linear map $c:V\to V$, $c(X):=XC=X-\frac1nX1_n1_n^T$. We record four properties.
\begin{enumerate}
    \item[(i)] $c(V)\subseteq\Vcent$ and $c(X)=X$ for $X\in\Vcent$. 
    \item[(ii)] $X$ and $c(X)$ lie in the same $H\times S_n$-orbit.
    \item[(iii)] $c$ is equivariant: $c\big((t,R,\sigma)\cdot X\big)=(0,R,\sigma)\cdot c(X)$ for all $(t,R,\sigma)\in H\times S_n$. 
    \item[(iv)] $\Vcent$ is $H_0\times S_n$-invariant: if $X1_n=0$ then $RXP_\sigma^T1_n=RX1_n=0$.
\end{enumerate}

\emph{Claim: $X,Y\in V$ lie in the same $H\times S_n$-orbit if and only if $c(X),c(Y)$ lie in the same $H_0\times S_n$-orbit.} 

If $Y=(t,R,\sigma)\cdot X$, then $c(Y)=(0,R,\sigma)\cdot c(X)$ by (iii). Conversely, if $c(Y)=(0,R,\sigma)\cdot c(X)$, then $c(X)$ and $c(Y)$ lie in the same $H\times S_n$-orbit, since $H_0\times S_n\subseteq H\times S_n$. By (ii), $X$ lies in the orbit of $c(X)$ and $Y$ lies in the orbit of $c(Y)$, so $X$ and $Y$ lie in the same $H\times S_n$-orbit.

We will repeatedly use that composing a polynomial with a linear map, or restricting it to a linear subspace, yields a polynomial of no larger degree.

We now show that $D(H\times S_n,V)= D(H_0\times S_n,\Vcent)$ by proving inequality in both directions
\emph{Step 1: $D(H\times S_n,V)\leq D(H_0\times S_n,\Vcent)$.} Set $D:=D(H_0\times S_n,\Vcent)$ and let $X,Y\in V$ lie in different $H\times S_n$-orbits. By the claim, $c(X),c(Y)\in\Vcent$ lie in different $H_0\times S_n$-orbits, so there is $p\in\RR[\Vcent]^{H_0\times S_n}_{\leq D}$ with $p(c(X))\neq p(c(Y))$. Set $q:=p\circ c\in\RR[V]_{\leq D}$. By (iii), for all $(t,R,\sigma)\in H\times S_n$,
\[
q\big((t,R,\sigma)\cdot X\big)=p\big((0,R,\sigma)\cdot c(X)\big)=p(c(X))=q(X),
\]
so $q$ is $H\times S_n$-invariant and separates $X$ from $Y$. Hence $\RR[V]^{H\times S_n}_{\leq D}$ is $(H\times S_n,V)$ complete.

\emph{Step 2: $D(H_0\times S_n,\Vcent)\leq D(H\times S_n,V)$.} Set $D':=D(H\times S_n,V)$ and let $X,Y\in\Vcent$ lie in different $H_0\times S_n$-orbits. By (i), $c(X)=X$ and $c(Y)=Y$, so by the claim $X,Y$ lie in different $H\times S_n$-orbits of $V$. Hence there is $q\in\RR[V]^{H\times S_n}_{\leq D'}$ with $q(X)\neq q(Y)$. Its restriction $q|_{\Vcent}$ is a polynomial on $\Vcent$ of degree $\leq D'$. It is $H_0\times S_n$-invariant, because $\Vcent$ is $H_0\times S_n$-invariant by (iv) and $H_0\times S_n\subseteq H\times S_n$. Also $q|_{\Vcent}(X)\neq q|_{\Vcent}(Y)$. Hence $\RR[\Vcent]^{H_0\times S_n}_{\leq D'}$ is $(H_0\times S_n,\Vcent)$ complete.

Finally, the inequality $D(H_0\times S_n,\Vcent)\leq D(H_0\times S_n,V)$ follows directly from the fact that $\Vcent$ is an invariant subspace of $V$.
\end{proof}

\graphs*
\begin{proof}
We first prove $D(O(n)\times S_n,\RR^{n\times n})\leq 2D(S_n,\Symn)$. Let $X,Y\in \RR^{n\times n}$ from different $O(n)\times S_n $ orbits. Then the Gram matrices $X^TX $ and $Y^TY $ are not in the same $S_n$ orbit, since otherwise, we would have that $X^TX=P^TY^TYP $ for some permutation matrix $P$, implying that $X $ and $YP$ have the same Gram matrix, and thus there exists an orthogonal matrix $R\in O(n)$ such that $X=RYP $, contradicting the assumption that $X,Y$ are not from the same orbit. It follows that there exists an $S_n$ invariant polynomial $p$ of degree $\leq  D(S_n,\Symn)$ such that $p(X^TX)\neq p(Y^TY) $. The mapping $X\mapsto p(X^TX) $ is a $O(n)\times S_n $ invariant polynomial of degree $2 \deg(p)$ and thus we have $D(O(n)\times S_n,\RR^{n\times n})\leq 2D(S_n,\Symn)$.

We now prove the converse direction $D(O(n)\times S_n,\RR^{n\times n})\geq 2D(S_n,\Symn)$. Let $A,B\in \Symn$ which are not from the same $S_n$ orbit. There exists some $t>0$ such that $A+tI_n $ and $B+tI_n $ are both positive definite. It follows that for some $X,Y\in \RR^{n\times n} $, 
$$A+tI_n=X^TX, B+tI_n=Y^TY.  $$
 $X,Y$ are not in the same $O(n)\times S_n $ orbit (for that would imply that $A,B$ are in the same $S_n$ orbit). Therefore there exists an $O(n)\times S_n $ invariant polynomial $p$ of degree $\leq D(O(n)\times S_n,\RR^{n\times n}) $, such that $p(X)\neq p(Y) $. By the First Fundamental Theorem of Invariant Theory for $O(n)$, any $O(n)$ invariant polynomial is a polynomial in the inner products of columns of $X$, or in other words, $p(X)=\phi(X^TX)$ for an appropriate polynomial $\phi$. The polynomial $\phi$ is $S_n$ invariant on all positive definite matrices, and since these form an open subset of $\Symn $, it follows that $\phi$ is $S_n$ invariant on all of $\Symn$. Finally we have that 
$$\phi(A+tI_n)=\phi(X^TX)\neq \phi(Y^TY)=\phi(B+tI_n) $$
and so the polynomial $A\mapsto \phi(A+tI_n) $ separates $A$ from $B$ and has degree 
$$\deg(\phi)=\frac{\deg(p)}{2} \leq \frac{ D(O(n)\times S_n,\RR^{n\times n})}{2} .$$ 
\end{proof}

\section{Upper Bound}

\subsection{Proofs of elementary statements}

\begin{proposition}\label{prop:gram-orbits}
    For $X,Y\in\Rdnmat$ we have $\Phi(X)=\Phi(Y)$ if and only if $Y=RX$ for some $R\in\OO(d)$.
\end{proposition}

\begin{proof}
    If $Y=RX$ with $R\in\OO(d)$, then $Y^TY=X^TR^TRX=X^TX$. Conversely, suppose $X^TX=Y^TY$, and let $x_1,\dots,x_n$ and $y_1,\dots,y_n$ denote the columns of $X$ and $Y$, so that $\langle x_i,x_j\rangle=\langle y_i,y_j\rangle$ for all $i,j$. The assignment $\sum_i c_ix_i\mapsto \sum_i c_iy_i$ is a well-defined linear isometry $T$ from $\mathrm{span}(x_1,\dots,x_n)$ onto $\mathrm{span}(y_1,\dots,y_n)$: for every $c\in\R^n$,
    \[ \Big\|\sum_i c_iy_i\Big\|^2=c^TY^TYc=c^TX^TXc=\Big\|\sum_i c_ix_i\Big\|^2, \]
    which gives both well-definedness (take $\sum_i c_ix_i=0$) and preservation of norms. Since $\ker Z=\ker(Z^TZ)$ for any real matrix $Z$, both spans have dimension $\mathrm{rank}(X^TX)=\mathrm{rank}(Y^TY)$, hence their orthogonal complements in $\R^d$ have equal dimension as well; extending $T$ by any linear isometry between these complements produces $R\in\OO(d)$ with $RX=Y$.
\end{proof}

\begin{proposition}[Equivariance of the Gram map]\label{prop:gram-equivariant}
    For all $X\in\Rdnmat$, $R\in\OO(d)$ and $\tau\in S_n$ we have
    \[ \Phi\big((R,\tau)\cdot X\big) \;=\; \tau\cdot \Phi(X). \]
    In particular, $\Phi$ is $\OO(d)$-invariant and intertwines the $S_n$-actions on $\Rdnmat$ and on $\Symn$.
\end{proposition}

\begin{proof}
    $(RXP_\tau^T)^T(RXP_\tau^T)=P_\tau X^TR^TRXP_\tau^T=P_\tau(X^TX)P_\tau^T$.
\end{proof}

\begin{lemma}[The first $r$ rows determine a rank-$r$ matrix]\label{lem:completion}
    Let $C\in \Symn$ be a rank-$r$ symmetric matrix whose leading $r\times r$ block $C_{11}$ is nonsingular. Then the first $r$ rows determine $C$; that is, writing
    \[C=\begin{pmatrix}
        C_{11} & C_{12} \\
        C_{12}^T & C_{22}
    \end{pmatrix}, \qquad C_{22} = C_{12}^TC_{11}^{-1}C_{12}.\]
\end{lemma}

\begin{proof}
    Setting $S = C_{22}-C_{12}^TC_{11}^{-1}C_{12}$, block multiplication verifies the congruence
    \[ C=\begin{pmatrix} I&0\\ C_{12}^TC_{11}^{-1} & I\end{pmatrix}
         \begin{pmatrix}C_{11}&0\\0&S\end{pmatrix}
         \begin{pmatrix}I& C_{11}^{-1}C_{12}\\0&I\end{pmatrix}, \]
    so that $\mathrm{rank}(C)=\mathrm{rank}(C_{11})+\mathrm{rank}(S)$ (rank additivity of the Schur complement, \citealp{guttman1946enlargement}). As both sides of the first equality are $r$, we get $\mathrm{rank}(S)=0$, i.e.\ $S=0$. Note that positive semidefiniteness is not needed.
\end{proof}

The point of Lemma~\ref{lem:completion} is that it localizes disagreement between non-equivalent matrices to the first $r$ rows, whose positions number
\[ |\mathcal{P}_r| = \sum_{i=1}^{r}(n-i+1) = rn-\binom{r}{2}=N_r\leq N_d, \qquad \mathcal{P}_r:=\{(i,j) : 1\leq i\leq r,\ i\leq j\leq n\}. \]

\begin{lemma}[Witness lemma]\label{lem:witness}
    Let $A,B\in \Symn$ be rank-$r$ matrices in different $S_n$-orbits, with the leading $r\times r$ block of $A$ nonsingular. Then for every $\sigma\in S_n$ there exists $(i,j)\in\mathcal{P}_r$ such that $(\sigma\cdot B)_{ij}\neq A_{ij}$.
\end{lemma}

\begin{proof}
    Suppose otherwise, so that $\sigma\cdot B$ agrees with $A$ on the first $r$ rows. In particular its leading $r\times r$ block equals that of $A$ and is therefore nonsingular, while $\mathrm{rank}(\sigma\cdot B)=\mathrm{rank}(B)=r$. Lemma~\ref{lem:completion} thus applies to both $A$ and $\sigma\cdot B$, and since they share the same first $r$ rows they share the same completion: $\sigma\cdot B=A$, contradicting that $A$ and $B$ lie in different orbits.
\end{proof}

\begin{lemma}[Row power sums detect rows up to reordering]\label{lem:row-multiset}
    For $u\in\R^n$ define
    \[ \mathrm{pw}(u) \;=\; \Big(\sum_{l=1}^n u_l,\ \sum_{l=1}^n u_l^2,\ \dots,\ \sum_{l=1}^n u_l^n\Big)\in\R^n, \]
    and for $C\in \Symn$ let $\mathrm{pw}_i(C):=\mathrm{pw}(C_{i1},\dots,C_{in})$, a polynomial map whose components have degree $\leq n$ in the entries of $C$. Then:
    \begin{enumerate}
        \item[(a)] $\mathrm{pw}(u)=\mathrm{pw}(v)$ if and only if $u$ and $v$ agree up to a permutation of their entries;
        \item[(b)] $\mathrm{pw}_i(\tau\cdot C)=\mathrm{pw}_{\tau^{-1}(i)}(C)$ for all $\tau\in S_n$ and $C\in \Symn$.
    \end{enumerate}
\end{lemma}

\begin{proof}
    (a) Power sums are symmetric in the entries, so agreement up to permutation implies $\mathrm{pw}(u)=\mathrm{pw}(v)$. Conversely, in characteristic zero Newton's identities \citep[Ch.~I, \S2]{macdonald1995symmetric} determine the elementary symmetric polynomials $e_1,\dots,e_n$ of $u_1,\dots,u_n$ from the power sums $p_1,\dots,p_n$; hence $\mathrm{pw}(u)$ determines the polynomial $\prod_{l=1}^n(t-u_l)=t^n-e_1t^{n-1}+\dots+(-1)^ne_n$, whose multiset of roots is $\{u_1,\dots,u_n\}$.
    (b) The $i$-th row of $\tau\cdot C$ is $\big(C_{\tau^{-1}(i),\tau^{-1}(j)}\big)_{j=1}^n$; as $j$ runs over $\{1,\dots,n\}$, so does $\tau^{-1}(j)$, so this row is a reordering of the $\tau^{-1}(i)$-th row of $C$, and power sums are invariant under reordering.
\end{proof}

We will use the fact that a rank-$r$ Gram matrix contains a nonsingular $r \times r$ \emph{principal} submatrix: writing $C=Z^TZ$ with $Z\in\R^{r\times n}$ and choosing $r$ linearly independent columns of $Z$, the principal submatrix indexed by those columns is their Gram matrix, which is nonsingular. Principality matters here, because $S_n$ acts by \emph{simultaneous} row and column permutations, so only a principal submatrix can be moved to the leading position. As invariant polynomials are constant on orbits, we may replace $A$ by any element of its orbit, and hence assume that the leading $r\times r$ block of $A$ is nonsingular. No assumption on $B$ is needed.

The \emph{Gram map} $\Phi:\Rdnmat\to \Symn$, $\Phi(X)=X^TX$, records the pairwise inner products of the points of the cloud, and its image is the cone $\mathcal{G}_d=\{C\in \Symn : C\succeq 0,\ \mathrm{rank}(C)\leq d\}$.

\begin{corollary}\label{cor:reduction}
    Two matrices $X,Y\in\Rdnmat$ lie in the same $\OO(d)\times S_n$-orbit if and only if $\Phi(X)$ and $\Phi(Y)$ lie in the same $S_n$-orbit in $\Symn$. Consequently, if $Q:\Symn\to\R$ is an $S_n$-invariant polynomial with $Q(\Phi(X))\neq Q(\Phi(Y))$, then $p:=Q\circ\Phi$ is an $\OO(d)\times S_n$-invariant polynomial with $p(X)\neq p(Y)$ and $\deg(p)\leq 2\deg(Q)$.
\end{corollary}

\begin{proof}
    If $Y=(R,\tau)\cdot X$, then $\Phi(Y)=\tau\cdot\Phi(X)$ by Proposition~\ref{prop:gram-equivariant}. Conversely, if $\Phi(Y)=\tau\cdot\Phi(X)=\Phi(XP_\tau^T)$, then Proposition~\ref{prop:gram-orbits} provides $R\in\OO(d)$ with $Y=R(XP_\tau^T)=(R,\tau)\cdot X$. For the second statement: $p$ is $\OO(d)$-invariant because $\Phi$ is, and $S_n$-invariant because $Q$ is and $\Phi$ intertwines the two $S_n$-actions (Proposition~\ref{prop:gram-equivariant}); every entry of $\Phi(X)$ is a polynomial of degree $2$ in the entries of $X$, so $\deg(p)\leq 2\deg(Q)$.
\end{proof}

\begin{remark}
    By Weyl's first fundamental theorem for $\OO(d)$ \citep[Theorem~2.9.A]{weyl1946classical}, this reduction is lossless: \emph{every} $\OO(d)\times S_n$-invariant polynomial on $\Rdnmat$ is a polynomial in the pairwise inner products $(X^TX)_{ij}$. For the upper bound below, however, only the elementary direction is used, namely Corollary~\ref{cor:reduction}.
\end{remark}

Corollary~\ref{cor:reduction} reduces Theorem~\ref{thm:dn2} to a statement about Gram matrices: it suffices to separate any two matrices $A,B\in\mathcal{G}_d$ lying in different $S_n$-orbits by an $S_n$-invariant polynomial of degree $\leq dn(2n-1)+N_d(n-1)$ in the Gram entries. We will introduce some preliminary lemmas and claims before embarking on a full proof.

\begin{lemma}\label{lem:rank-sep}
    If $\mathrm{rank}(A)\neq \mathrm{rank}(B)$ then there exists $k\leq n$ such that $\mathrm{tr}(A^k) \neq \mathrm{tr}(B^k)$. Therefore, a polynomial of degree $\leq n$ can separate $A$ and $B$.
\end{lemma}

Proof of this lemma follows from the power sum polynomials being a determining set for the elementary symmetric polynomials in characteristic zero (Newton's identities, \citealp[Ch.~I, \S2]{macdonald1995symmetric}), combined with the trace of the $k$-th power of a symmetric matrix being the $k$-th power sum of its eigenvalues: if $\mathrm{tr}(A^k)=\mathrm{tr}(B^k)$ for all $k\leq n$ then $A$ and $B$ have the same characteristic polynomial, hence the same number of nonzero eigenvalues. The power sums are of maximal degree $n$, and each $\mathrm{tr}(C^k)$ is $S_n$-invariant since conjugation by a permutation matrix is a similarity. Since rank is likewise constant on orbits, from now on we may assume $\mathrm{rank}(A)=\mathrm{rank}(B)=r\leq d$. More generally, if $A$ and $B$ are not similar, then there exists $k\leq n$ such that $\mathrm{tr}(A^k)\neq \mathrm{tr}(B^k)$, since symmetric matrices are diagonalizable, so two of them are similar if and only if they have the same characteristic polynomial.

\begin{lemma}[Row indicators]\label{lem:row-indicator}
    Let $A,B\in \Symn$ and $1\leq i\leq n$. There exists a polynomial $\delta_i:\Symn\to\R$ with $\deg(\delta_i)\leq n(2n-1)$ such that for every $\tau\in S_n$ and every $C\in\{A,B\}$,
    \[ \delta_i(\tau\cdot C)=\begin{cases} 1 & \text{if the $i$-th row of $\tau\cdot C$ agrees with the $i$-th row of $A$ up to reordering},\\ 0 & \text{otherwise.}\end{cases} \]
\end{lemma}

\begin{proof}
    Let $\mathcal{N}=\{\mathrm{pw}_j(A) : 1\leq j\leq n\}\cup\{\mathrm{pw}_j(B) : 1\leq j\leq n\}\subset\R^n$, a set of $m:=|\mathcal{N}|\leq 2n$ distinct vectors. There is a linear functional $\pi:\R^n\to\R$ that is injective on $\mathcal{N}$: for each pair $v\neq w$ in $\mathcal{N}$, the functionals with $\pi(v)=\pi(w)$ form a proper linear subspace of the dual space $(\R^n)^*$, and a vector space over an infinite field is not a union of finitely many proper subspaces. Let $L_i$ be the univariate Lagrange polynomial of degree $\leq m-1\leq 2n-1$ taking the value $1$ at $\pi(\mathrm{pw}_i(A))$ and $0$ at the $m-1$ remaining points of $\pi(\mathcal{N})$, and set
    \[ \delta_i \;:=\; L_i\circ\pi\circ\mathrm{pw}_i . \]
    Since $\pi\circ\mathrm{pw}_i$ is a polynomial of degree $\leq n$ in the entries of $C$, we get $\deg(\delta_i)\leq n(2n-1)$.

    Now let $\tau\in S_n$ and $C\in\{A,B\}$. By Lemma~\ref{lem:row-multiset}(b), $\mathrm{pw}_i(\tau\cdot C)=\mathrm{pw}_{\tau^{-1}(i)}(C)\in\mathcal{N}$, so $\delta_i(\tau\cdot C)\in\{0,1\}$; as $\pi$ is injective on $\mathcal{N}$, the value is $1$ exactly when $\mathrm{pw}_i(\tau\cdot C)=\mathrm{pw}_i(A)$, which by Lemma~\ref{lem:row-multiset}(a) happens exactly when the $i$-th row of $\tau\cdot C$ agrees with the $i$-th row of $A$ up to reordering.
\end{proof}

We fix the conventions used below. On $\Symn$ the action is $(\tau\cdot C)_{ij}=C_{\tau^{-1}(i)\tau^{-1}(j)}$, and on $\Rdnmat$ it is $(R,\tau)\cdot Z=RZP_\tau^T$ (Proposition~\ref{prop:gram-equivariant}), where $P_\tau$ is the permutation matrix with $P_\tau e_j=e_{\tau(j)}$, so that $P_{\tau\sigma}=P_\tau P_\sigma$. We write $\tau\cdot Z:=ZP_\tau^T$ and $RZ$ for the actions of the two factors, so that $(R,\tau)\cdot Z=R(\tau\cdot Z)=\tau\cdot(RZ)$. The $j$-th column of $\tau\cdot Z$ is $z_{\tau^{-1}(j)}$: this is the action of the Notation section, and it is the convention under which Proposition~\ref{prop:gram-equivariant} holds, since $\Phi(\tau\cdot Z)_{ij}=\langle z_{\tau^{-1}(i)},z_{\tau^{-1}(j)}\rangle=(\tau\cdot\Phi(Z))_{ij}$. Consequently $\tau\cdot(\sigma\cdot Z)=(\tau\sigma)\cdot Z$ on $\Rdnmat$, and likewise on $\Symn$. Moreover $\mathrm{rank}(\Phi(Z))=\mathrm{rank}(Z)$, because $\ker Z=\ker(Z^TZ)$. Finally, for $d\leq n$ and $Z\in\Rdnmat$ we let $Z_{[d]}\in\R^{d\times d}$ denote the matrix formed by the first $d$ columns of $Z$. Thus $(\tau\cdot Z)_{[d]}$ has columns $z_{\tau^{-1}(1)},\dots,z_{\tau^{-1}(d)}$, and $(\tau\cdot Z)_{[d]}^T(\tau\cdot Z)_{[d]}$ is the leading $d\times d$ block of $\Phi(\tau\cdot Z)=\tau\cdot\Phi(Z)$.

\begin{lemma}[Rank-deficient clouds]\label{lem:rank-deficient}
    Let $X\in\Rdnmat$ with $\mathrm{rank}(X)<d$. Then the $\SO(d)\times S_n$-orbit of $X$ equals its $\OO(d)\times S_n$-orbit. In particular, if $Y\in\Rdnmat$ lies in the $\OO(d)\times S_n$-orbit of $X$, then $Y$ lies in the $\SO(d)\times S_n$-orbit of $X$.
\end{lemma}

\begin{proof}
    The columns of $X$ span a subspace $U\subseteq\R^d$ of dimension $\mathrm{rank}(X)<d$, so there is a unit vector $u\in\R^d$ orthogonal to $U$. Let $M:=I_d-2uu^T$ be the reflection in the hyperplane $u^\perp\supseteq U$. Then $M^TM=I_d$, so $M\in\OO(d)$; $Mu=-u$ and $Mv=v$ for $v\perp u$, so $\det(M)=-1$; and $Mx=x$ for every $x\in U$, so $MX=X$. Now let $(R,\tau)\in\OO(d)\times S_n$. If $\det(R)=1$ there is nothing to show. If $\det(R)=-1$, then
    \[ (R,\tau)\cdot X \;=\; RXP_\tau^T \;=\; RMXP_\tau^T \;=\; (RM,\tau)\cdot X, \qquad \det(RM)=1 . \]
    Hence every element of the $\OO(d)\times S_n$-orbit of $X$ lies in its $\SO(d)\times S_n$-orbit; the reverse inclusion is trivial.
\end{proof}

\begin{lemma}[Stabiliser sign]\label{lem:stab-sign}
    Let $d\leq n$, let $X\in\Rdnmat$ have rank $d$, and set $A:=\Phi(X)$. Let $\tau\in S_n$ satisfy $\tau\cdot A=A$.
    \begin{enumerate}
        \item[(a)] There is a unique $R_\tau\in\OO(d)$ with $\tau\cdot X=R_\tau X$, and $\det (\tau\cdot X)_{[d]}=\det(R_\tau)\det X_{[d]}$.
        \item[(b)] For every $M\in\OO(d)$ we have $\Phi(MX)=A$ and $\det (\tau\cdot MX)_{[d]}=\det(M)\det(R_\tau)\det X_{[d]}$.
        \item[(c)] If $\det(R_\tau)=-1$, then $MX$ lies in the $\SO(d)\times S_n$-orbit of $X$ for every $M\in\OO(d)$ with $\det(M)=-1$.
    \end{enumerate}
\end{lemma}

\begin{proof}
    (a) By Proposition~\ref{prop:gram-equivariant}, $\Phi(\tau\cdot X)=\tau\cdot\Phi(X)=\tau\cdot A=A=\Phi(X)$, so Proposition~\ref{prop:gram-orbits} provides $R_\tau\in\OO(d)$ with $\tau\cdot X=R_\tau X$. If also $\tau\cdot X=R'X$ with $R'\in\OO(d)$, then $(R_\tau-R')X=0$, so $R_\tau-R'$ vanishes on the span of the columns of $X$, which is $\R^d$ because $\mathrm{rank}(X)=d$; hence $R'=R_\tau$. The first $d$ columns of $R_\tau X$ are $R_\tau$ times the first $d$ columns of $X$, so $(\tau\cdot X)_{[d]}=(R_\tau X)_{[d]}=R_\tau X_{[d]}$, and $\det$ is multiplicative.

    (b) $\Phi(MX)=\Phi(X)=A$ by Proposition~\ref{prop:gram-equivariant}. The two actions commute: $\tau\cdot(MX)=MXP_\tau^T=M(\tau\cdot X)=MR_\tau X$ by (a). Hence $(\tau\cdot MX)_{[d]}=MR_\tau X_{[d]}$ and $\det (\tau\cdot MX)_{[d]}=\det(M)\det(R_\tau)\det X_{[d]}$.

    (c) By (a), $MX=MR_\tau^{-1}(R_\tau X)=MR_\tau^{-1}(\tau\cdot X)=MR_\tau^{-1}XP_\tau^T=(MR_\tau^{-1},\tau)\cdot X$, and $\det(MR_\tau^{-1})=\det(M)/\det(R_\tau)=1$.
\end{proof}

\begin{proof}[Proof of Theorem~\ref{thm:dn2}]
    For $1\leq r\leq d$ set $c_r:=rn(2n-1)+N_r(n-1)$. We show that any two $X,Y\in\Rdnmat$ in different $\OO(d)\times S_n$-orbits are separated by an $\OO(d)\times S_n$-invariant polynomial of degree $\leq 2c_d$, and that any two in different $\SO(d)\times S_n$-orbits are separated by an $\SO(d)\times S_n$-invariant polynomial of degree $\leq 2c_d+d$. The two bounds of the theorem then follow from the identity $2c_d=6dn^2-4dn-d(d-1)(n-1)$ verified at the end.

    \emph{Reductions.} Let $X,Y\in\Rdnmat$ lie in different $\OO(d)\times S_n$-orbits, or in different $\SO(d)\times S_n$-orbits, and set $A_0:=\Phi(X)$, $B:=\Phi(Y)$. If $\mathrm{rank}(A_0)\neq\mathrm{rank}(B)$, Lemma~\ref{lem:rank-sep} gives $k\leq n$ with $\mathrm{tr}(A_0^k)\neq\mathrm{tr}(B^k)$. As $C\mapsto\mathrm{tr}(C^k)$ is an $S_n$-invariant polynomial on $\Symn$ of degree $k$, Corollary~\ref{cor:reduction} shows that $Z\mapsto\mathrm{tr}(\Phi(Z)^k)$ is an $\OO(d)\times S_n$-invariant, hence also $\SO(d)\times S_n$-invariant, polynomial of degree $\leq 2n$ separating $X$ from $Y$. So assume $\mathrm{rank}(A_0)=\mathrm{rank}(B)=:r$. Then $r=\mathrm{rank}(X)=\mathrm{rank}(Y)\leq\min(d,n)$, and $r\geq1$, since $\Phi(X)=0$ forces $X=0$ and $r=0$ would give $X=Y=0$. Choose $r$ linearly independent columns $x_{j_1},\dots,x_{j_r}$ of $X$ and $\sigma_0\in S_n$ with $\sigma_0(j_i)=i$ for $i\leq r$. The $i$-th column of $\sigma_0\cdot X$ is $x_{\sigma_0^{-1}(i)}=x_{j_i}$, so the leading $r\times r$ block of $\Phi(\sigma_0\cdot X)=\sigma_0\cdot A_0$ (Proposition~\ref{prop:gram-equivariant}) is the Gram matrix of $x_{j_1},\dots,x_{j_r}$, which is nonsingular. Since $\sigma_0\cdot X$ lies in the $S_n$-orbit of $X$, replacing $X$ by $\sigma_0\cdot X$ changes neither the hypothesis on the pair $(X,Y)$ nor the value at $X$ of any $\SO(d)\times S_n$-invariant polynomial. We may therefore assume that $A:=\Phi(X)$ has a nonsingular leading $r\times r$ block $A_{11}$, where $B=\Phi(Y)$ and $\mathrm{rank}(A)=\mathrm{rank}(B)=r\geq1$.

    \emph{Step 1: an $S_n$-invariant polynomial on $\Symn$.}
    Let $\delta_1,\dots,\delta_r$ be the row indicators of Lemma~\ref{lem:row-indicator} for the pair $(A,B)$; each has degree $\leq n(2n-1)$. For $(i,j)\in\mathcal{P}_r$ let
    \[ \Lambda_i \;:=\; \{A_{il} : 1\leq l\leq n\} \]
    be the set of values in the $i$-th row of $A$, and let $\ell_{ij}$ be the univariate Lagrange polynomial with $\ell_{ij}(A_{ij})=1$ and $\ell_{ij}(\lambda)=0$ for all $\lambda\in\Lambda_i\setminus\{A_{ij}\}$; its degree is $\leq|\Lambda_i|-1\leq n-1$. For $C\in \Symn$ define
    \[ \chi(C) \;:=\; \prod_{i=1}^{r}\delta_i(C)\ \prod_{(i,j)\in\mathcal{P}_r}\ell_{ij}\big(C_{ij}\big), \qquad Q(C) \;:=\; \sum_{\tau\in S_n} \chi(\tau\cdot C) . \]
    The polynomial $\chi$ has degree $\leq\sum_{i\leq r}\deg(\delta_i)+\sum_{(i,j)\in\mathcal{P}_r}\deg(\ell_{ij})\leq rn(2n-1)+N_r(n-1)=c_r$, and since $C\mapsto\tau\cdot C$ is linear, so does $Q$. For $\sigma\in S_n$ we have $\tau\cdot(\sigma\cdot C)=(\tau\sigma)\cdot C$, hence $\chi(\tau\cdot(\sigma\cdot C))=\chi((\tau\sigma)\cdot C)$; as $\tau\mapsto\tau\sigma$ is a bijection of $S_n$, replacing $C$ by $\sigma\cdot C$ merely reindexes the sum, and $Q$ is $S_n$-invariant.

    We claim that for $C\in\{A,B\}$ and $\tau\in S_n$,
    \begin{equation}\tag{$\ast$}
        \chi(\tau\cdot C)\in\{0,1\}, \qquad\text{and}\qquad \chi(\tau\cdot C)=1 \iff \tau\cdot C=A .
    \end{equation}
    Fix such $C$ and $\tau$. If $\delta_i(\tau\cdot C)=0$ for some $i\leq r$, then $\chi(\tau\cdot C)=0$ (an exact zero annihilates the product), and $\tau\cdot C\neq A$, because $\delta_i(A)=1$ by Lemma~\ref{lem:row-indicator}. Otherwise $\delta_i(\tau\cdot C)=1$ for all $i\leq r$, so by Lemma~\ref{lem:row-indicator} every row $i\leq r$ of $\tau\cdot C$ is a reordering of row $i$ of $A$. Then $(\tau\cdot C)_{ij}\in\Lambda_i$ for every $(i,j)\in\mathcal{P}_r$, where $\ell_{ij}$ takes values in $\{0,1\}$, with value $1$ exactly when $(\tau\cdot C)_{ij}=A_{ij}$. (Off its node set a Lagrange factor is uncontrolled --- possibly negative or large --- which is why the exact vanishing of a $\delta$-factor is what disposes of the mismatched terms.) Hence $\chi(\tau\cdot C)\in\{0,1\}$, and $\chi(\tau\cdot C)=1$ if and only if $(\tau\cdot C)_{ij}=A_{ij}$ for all $(i,j)\in\mathcal{P}_r$, that is, by symmetry, if and only if $\tau\cdot C$ and $A$ have the same first $r$ rows. If $\tau\cdot C=A$ this holds. Conversely, if the first $r$ rows agree, then $\tau\cdot C$ and $A$ share the nonsingular leading block $A_{11}$, and both have rank $r$ (conjugation by a permutation matrix preserves rank), so Lemma~\ref{lem:completion} applies to both and yields the same completion: $\tau\cdot C=A$. This proves ($\ast$).

    Writing $\mathrm{Stab}(A):=\{\tau\in S_n:\tau\cdot A=A\}$, ($\ast$) gives
    \[ Q(A)=|\mathrm{Stab}(A)|\;\geq\;1, \qquad Q(B)=\big|\{\tau\in S_n:\ \tau\cdot B=A\}\big| , \]
    so $Q(B)=0$ whenever $A$ and $B$ lie in different $S_n$-orbits (cf.\ Lemma~\ref{lem:witness}). The polynomial $Q$ is used in Steps~2 and~3, the polynomial $\chi$ in Step~4.

    \emph{Step 2: $\OO(d)\times S_n$.}
    Let $X,Y$ lie in different $\OO(d)\times S_n$-orbits. By Corollary~\ref{cor:reduction}, $A$ and $B$ lie in different $S_n$-orbits, so $Q(A)\geq1>0=Q(B)$ by Step~1. By Corollary~\ref{cor:reduction} again, $Q\circ\Phi$ is an $\OO(d)\times S_n$-invariant polynomial of degree $\leq 2c_r\leq 2c_d$ with $Q(\Phi(X))=Q(A)\neq Q(B)=Q(\Phi(Y))$.

    \emph{Step 3: $\SO(d)\times S_n$, $r<d$.}
    Let $X,Y$ lie in different $\SO(d)\times S_n$-orbits, with $r<d$. If $Y$ lay in the $\OO(d)\times S_n$-orbit of $X$, Lemma~\ref{lem:rank-deficient}, applicable because $\mathrm{rank}(X)=r<d$, would place $Y$ in the $\SO(d)\times S_n$-orbit of $X$. Hence $X,Y$ lie in different $\OO(d)\times S_n$-orbits, and the polynomial $Q\circ\Phi$ of Step~2 separates them. Being $\OO(d)\times S_n$-invariant, it is $\SO(d)\times S_n$-invariant, and its degree is $\leq 2c_d$.

    \emph{Step 4: $\SO(d)\times S_n$, $r=d$.}
    Let $X,Y$ lie in different $\SO(d)\times S_n$-orbits, with $r=d$; then $d\leq n$ and $X$ has rank $d$. By the conventions above, $X_{[d]}^TX_{[d]}=A_{11}$, so $(\det X_{[d]})^2=\det A_{11}\neq0$. Let
    \[ \kappa(t) \;:=\; \tfrac12\Big(1+\frac{t}{\det X_{[d]}}\Big) \]
    be the affine function with $\kappa(\det X_{[d]})=1$ and $\kappa(-\det X_{[d]})=0$, and define, for $Z\in\Rdnmat$,
    \begin{align*}
    p(Z) \;&:=\; \sum_{\tau\in S_n}\ \kappa\big(\det (\tau\cdot Z)_{[d]}\big)\ \chi\big(\Phi(\tau\cdot Z)\big)\\
    &\;=\; \sum_{\tau\in S_n} \kappa\big(\det (\tau\cdot Z)_{[d]}\big)\, \chi\big(\tau\cdot\Phi(Z)\big) ,
    \end{align*}
    where the second expression uses $\Phi(\tau\cdot Z)=\tau\cdot\Phi(Z)$ (Proposition~\ref{prop:gram-equivariant}). Thus $p$ is a polynomial on $\Rdnmat$ whose $\tau$-summand is the $\tau$-summand of $Q\circ\Phi$ multiplied by a determinant factor; the latter does not factor through $\Phi$.

    \emph{Invariance.} In the first expression, the $\tau$-summand depends on $Z$ only through $\tau\cdot Z$. For $\sigma\in S_n$ we have $\tau\cdot(\sigma\cdot Z)=(\tau\sigma)\cdot Z$, so the $\tau$-summand of $p(\sigma\cdot Z)$ is the $\tau\sigma$-summand of $p(Z)$, and $p(\sigma\cdot Z)=p(Z)$. For $R\in\OO(d)$ we have $\tau\cdot(RZ)=R(\tau\cdot Z)$, hence $\Phi(\tau\cdot RZ)=\Phi(\tau\cdot Z)$ by Proposition~\ref{prop:gram-equivariant} and $(\tau\cdot RZ)_{[d]}=R\,(\tau\cdot Z)_{[d]}$, so that
    \[ p(RZ) \;=\; \sum_{\tau\in S_n} \kappa\big(\det(R)\det (\tau\cdot Z)_{[d]}\big)\, \chi\big(\tau\cdot\Phi(Z)\big) . \]
    For $\det(R)=1$ this is $p(Z)$. Hence $p$ is $\SO(d)\times S_n$-invariant. For $\det(R)=-1$ the argument of $\kappa$ changes sign, so $p$ need not be $\OO(d)$-invariant; this is the only place where $p$ can distinguish a cloud from its mirror image.

    \emph{Values.} Let $Z\in\Rdnmat$ satisfy $\Phi(Z)=A$. By ($\ast$), $\chi(\tau\cdot A)=1$ for $\tau\in\mathrm{Stab}(A)$ and $\chi(\tau\cdot A)=0$ otherwise, so
    \[ p(Z) \;=\; \sum_{\tau\in\mathrm{Stab}(A)} \kappa\big(\det (\tau\cdot Z)_{[d]}\big) . \]
    For $Z=X$, Lemma~\ref{lem:stab-sign}(a) gives $\det (\tau\cdot X)_{[d]}=\det(R_\tau)\det X_{[d]}$ with $\det(R_\tau)=\pm1$, so each term is $\kappa(\pm\det X_{[d]})\in\{0,1\}$. Fix $M\in\OO(d)$ with $\det(M)=-1$, for instance $M=\mathrm{diag}(-1,1,\dots,1)$. By Lemma~\ref{lem:stab-sign}(b), $\Phi(MX)=A$ and $\det (\tau\cdot MX)_{[d]}=-\det(R_\tau)\det X_{[d]}$. Therefore
    \begin{align*}
    p(X) \;&=\; \#\{\tau\in\mathrm{Stab}(A):\ \det(R_\tau)=+1\}\;\geq\;1,\\
    p(MX) \;&=\; \#\{\tau\in\mathrm{Stab}(A):\ \det(R_\tau)=-1\} ,
    \end{align*}
    where $p(X)\geq1$ because $\mathrm{id}\in\mathrm{Stab}(A)$ and $R_{\mathrm{id}}=I_d$ by the uniqueness in Lemma~\ref{lem:stab-sign}(a). Finally, if $A$ and $B$ lie in different $S_n$-orbits, then $\chi(\tau\cdot B)=0$ for every $\tau$ by ($\ast$), and $\Phi(Y)=B$ gives $p(Y)=0$.

    \emph{Separation.} If $Y$ lies outside the $\OO(d)\times S_n$-orbit of $X$, then $A$ and $B$ lie in different $S_n$-orbits by Corollary~\ref{cor:reduction}, and $p(X)\geq1>0=p(Y)$. Otherwise $Y=(R',\sigma)\cdot X$ with $R'\in\OO(d)$ and $\sigma\in S_n$. Every such $Y$ has the form $(R,\sigma)\cdot X$ or $(R,\sigma)\cdot MX$ with $R\in\SO(d)$: take $R=R'$ if $\det(R')=1$, and otherwise $R=R'M^{-1}\in\SO(d)$, for which $(R,\sigma)\cdot MX=R'M^{-1}MXP_\sigma^T=Y$. The first form is excluded, since $Y$ is not in the $\SO(d)\times S_n$-orbit of $X$. So $Y=(R,\sigma)\cdot MX$ with $R\in\SO(d)$, and $p(Y)=p(MX)$ by invariance. Moreover $Y$ and $MX$ lie in the same $\SO(d)\times S_n$-orbit, so $MX$ does not lie in the $\SO(d)\times S_n$-orbit of $X$; by Lemma~\ref{lem:stab-sign}(c), no $\tau\in\mathrm{Stab}(A)$ has $\det(R_\tau)=-1$, and the value formula gives $p(MX)=0$. Thus $p(X)\geq1>0=p(Y)$ in this case too. In either case $p$ is an $\SO(d)\times S_n$-invariant polynomial separating $X$ from $Y$.

    \emph{Degree.}\label{pg:degree} Since $N_{r+1}-N_r=n-r$, we have $c_{r+1}-c_r=n(2n-1)+(n-1)(n-r)>0$ for $r<n$, so $c_r$ is increasing in $r$ and $c_r\leq c_d=dn(2n-1)+N_d(n-1)$ for all $r\leq d$; also $2n\leq 2c_1\leq 2c_d$. Hence the polynomials of the rank reduction and of Steps~2 and~3 have degree $\leq 2c_d$. In Step~4, every entry of $\Phi(Z)$ is quadratic in the entries of $Z$, so $Z\mapsto\chi(\Phi(\tau\cdot Z))$ has degree $\leq 2c_d$; and $\kappa$ is affine while $\det (\tau\cdot Z)_{[d]}$, the determinant of the $d\times d$ matrix with columns $z_{\tau^{-1}(1)},\dots,z_{\tau^{-1}(d)}$, is a homogeneous polynomial of degree $d$ in the entries of $Z$. Hence $\deg(p)\leq 2c_d+d$. Using $2N_d(n-1)=2dn(n-1)-d(d-1)(n-1)=2dn^2-2dn-d(d-1)(n-1)$,
    \begin{align*}
    2c_d \;&=\; 2dn(2n-1)+2N_d(n-1)\\
    &=\; (4dn^2-2dn)+(2dn^2-2dn)-d(d-1)(n-1) \;=\; 6dn^2-4dn-d(d-1)(n-1) .
    \end{align*}
    Therefore $D(\OO(d)\times S_n,\Rdnmat)\leq 6dn^2-4dn-d(d-1)(n-1)$ by Step~2, and $D(\SO(d)\times S_n,\Rdnmat)\leq 6dn^2-4dn-d(d-1)(n-1)+d$ by Steps~3 and~4. By Lemma~\ref{lem:closing_options}, $D(E(d)\times S_n,\Rdnmat)\leq D(\OO(d)\times S_n,\Rdnmat)$ and $D(SE(d)\times S_n,\Rdnmat)\leq D(\SO(d)\times S_n,\Rdnmat)$, so the same two bounds hold for $E(d)$ and $SE(d)$. Both bounds are at most $6dn^2$, because $d\leq 4dn+d(d-1)(n-1)$ for $n\geq1$.
\end{proof}

\section{Constructing invariant polynomials via invariant models}
\subsection{Tensor Field Networks} \label{app:tfn}
Here we add some details on the construction of $\Gthree$ invariant polynomials by  TFN (Proposition \ref{prop:TFN}). First we discuss our assumptions: the TFN model we refer to is the one defined in \cite{dymMaron}. It  differs from the original TFN models only in that it allows the radial functions in the TFN model to be polynomial rather than MLPs. This allows the TFN models to exactly express polynomials. In the MLP version, the word express can be replaced with 'approximate on a compact set', since MLPs can approximate any polynomial on a compact set. This is discussed in detail in \cite{dymMaron}. 

In our analysis we assume that TFN produces a permutation invariant global feature by summing over all node features. 

Now we  explain how Proposition \ref{prop:TFN} follows from the results in \cite{dymMaron}. The first part of Theorem 2 in that paper, in the special case where $W_T$ is the trivial representation $W_0=\RR $, gives the following result:
\begin{proposition}  [Taken from Theorem 2 in \citealp{dymMaron}]
For all $n,D\in \NN$, every polynomial $p:\RR^{3\times n} \to \RR^n $ which is rotation and translation invariant, and permutation equivariant, can be expressed using TFN with irreducible features of order $D$, and $2D$ layers. 
\end{proposition}
It remains only to show that this claim can be translated from equivariant polynomials to invariant polynomials. This is easily done: let $q:\RR^{3\times n} \to \RR $ be a $\Gthree $
invariant polynomial of degree $D$. Then
$p(X)=n^{-1}q(X)1_n $
is a rotation and translation invariant, and permutation equivariant, polynomial. And therefore $p$ can be expressed as a TFN model of the advertised size. By applying sum pooling we obtain $q$, that is 
$q(X)=\sum_{j=1}^n p_j(X) .$

\subsection{MACE}\label{app:MACE}
Here we give the details regarding the construction of $\Gthree$ invariant polynomials by  MACE (Proposition \ref{prop:mace}).

We consider the original MACE \citep{mace} architecture with the following simplifications: (a) we take a single iteration $t=1$ (b) we remove from the input the $z_i$ features discussed in MACE (representing the atom numbers) (c) For every point $i$ we perform message passing on all points and not only on points within a given distance threshold.  We also make two changes (a) as in the TFN model, we allow the functions $R_{k\ell} $ to be any continuous functions, and (b) following ACE \citep{ace}, in which each factor of a basis function has its own radial index, we use weights $w_{k\tilde k\xi}$ that depend on the position $\xi$ of the factor, instead of $w_{k\tilde k\ell_\xi}$ of \cite{mace}. This extends MACE: with $w_{k\tilde k\ell_\xi}$, factors of equal order in the same channel coincide, so couplings antisymmetric in them vanish, e.g. the triple product $\det(v,v,v)=0$ of three equal factors of order $1$. In particular, if $w_{k\tilde k\xi}=1$ for $\tilde k=k_\xi$ and $w_{k\tilde k\xi}=0$ otherwise, as in the proof below, then \eqref{eq:B} becomes $\sum_{\boldsymbol{\ell m}}\mathcal C^{LM}_{\eta_\nu,\boldsymbol{\ell m}}\prod_{\xi=1}^{\nu}A_{i,k_\xi\ell_\xi m_\xi}$, which has the form of an ACE basis function \citep{ace}, with the channel $k_\xi$ in the role of the radial index of the $\xi$-th factor. Finally, the Clebsch--Gordan coefficients in \eqref{eq:B}, and in \eqref{eq:At} below, range over all couplings of representations of $SO(3)$, without the parity selection rule of the $O(3)$-equivariant implementation of \cite{mace}, which makes the model invariant under $X\mapsto-X$; with this rule, the same argument gives $(E(3)\times S_n,\RR^{3\times n})$ completeness with $D=18n^2-18n+6$, since only couplings with $\ell_1+\dots+\ell_\nu$ even are then needed. 

With these changes, the MACE model is of the form:

\begin{align}
A_{i,k\ell m} &= \sum_{j\neq i} R_{k\ell}(|x_j-x_i|)\, Y^{m}_{\ell}\left(\frac{x_j-x_i}{|x_j-x_i|}\right) \label{eq:A} \\
  B_{i,\eta_\nu k L M} &= \sum_{\boldsymbol{\ell m}}
    \mathcal C^{LM}_{\eta_\nu, \boldsymbol{\ell m}}
    \prod_{\xi=1}^{\nu} \sum_{\tilde k} w_{k \tilde k \xi}\, A_{i,\tilde k \ell_\xi m_\xi},
  \qquad \boldsymbol{\ell m} = (\ell_1 m_1, \dots, \ell_\nu m_\nu) \label{eq:B}\\
  m_{i,kLM} &= \sum_{\nu} \sum_{\eta_\nu} W_{ k L, \eta_\nu}\, B_{i,\eta_\nu k L M}, \label{eq:pooling}
\\
E_i&=\mathrm{MLP}\left(m_{ik00}, k=1,\ldots,\Kmax \right) \label{eq:Ei}\\
E&=\sum_{i=1}^nE_i \label{eq:E}
\end{align}
Important hyperparameters of this model are: $\ellmax$, the maximal irreducible degree used throughout the model. $\numax$, the maximal `body order', and  $\Kmax$, the number of `channels'.  

We now explain this in more detail. In \eqref{eq:A} we define $\Kmax  $ high dimensional features $A_{i,k}, k=1,\ldots,\Kmax $
for each point $x_i$. Each such feature is in $W_0\oplus W_1\oplus \ldots \oplus W_{\ellmax} $, and so is indexed by $A_{i,k,\ell,m}$ where $\ell$ denotes the representation and $m$ is the internal index inside the representation. The functions $R_{k\ell} $ are generally chosen to be some parametric model, and in our analysis we will choose them to be polynomial functions, as will be seen in the proof.

In Equation \ref{eq:B} new features $B_{i,k} $ are obtained by taking tensor products of $\nu\leq \numax $ linear combinations of $A_{i,\tilde k} $ features, and the using the Clebsch-Gordan coefficients  $\mathcal C^{LM}_{\eta_\nu, \boldsymbol{lm}}$ to decompose these tensor products back to $W_0^{r_0}\oplus W_1^{r_1}\oplus \ldots \oplus W_{\ellmax}^{r_{\ellmax}}$ (the tensor product could have components in $W_\ell $ with $\ell>\ellmax$ but these components are discarded). The features $ B_{i,\eta_\nu k L M}  $ are indexed similarly to the $A_{i,k\ell m}$ features, but now we have an additional index $\eta_\nu$ due to the possible multiplicity $r_\ell>1 $ of the features. 

Equation \ref{eq:pooling} is a linear pooling step over multiplicities and body orders. As can be seen in the next equation, in our setting only the $L=0,m=0$ components are of interest in this step. 

Equation \ref{eq:Ei} obtains a rotation invariant feature per node $E_i$ by applying an MLP (Multi-Layer Perceptron) to the $\ell=0,m=0$ features, for which the action of $SO(3)$ is trivial, namely $D_\ell(R)=1 $ when $\ell=0$. 

Finally, Equation \ref{eq:E} produces a permutation invariant feature by summing over $E_i$. Thus $E=E(X) $ is a $\Gthree$ invariant function.
We now recall and prove Proposition \ref{prop:mace}:
\propmace*
In other words, this means that if we choose the MACE hyperparameters to be $\numax=\min\{D,n-1\} $ and $\ellmax = \lfloor D/2 \rfloor$, and choose some appropriate $\Kmax=\Kmax(D) $, then there exists a choice of the functions $R_{kl} $ and their parameters $w_{k \tilde k \eta},W_{kL\eta_\nu} $ of MACE such that $E(X)=q(X)$.

The proof of the proposition will be based on the following proposition:
\begin{proposition}\label{prop:mace_equi}
Let $p:\RR^{3\times n}\to \RR^n$ be a polynomial with degree at most $D$ that is invariant to rotations and translations, and equivariant to permutations. That is, that the coordinates $p_j(X) $ are rotation and translation invariant, and
    \begin{equation}\label{eq:poly}
    \left(p_1(\tau(X)),\ldots,p_n(\tau(X)) \right)=\left(p_{\tau^{-1}(1)}(X),\ldots, p_{\tau^{-1}(n)}(X) \right), \quad \forall \tau \in S_n. \end{equation}
  Then there exists a single layer MACE architecture with $\numax=\min\{D,n-1\}$ and $\ellmax = \lfloor D/2 \rfloor$, such that the node features $E_i=E_i(X) $ from the MACE architecture satisfy
  $$E_i(X)=p_i(X), \forall i=1\ldots,n $$
\end{proposition}

This proposition easily implies the original proposition \ref{prop:mace}, for if $q:\RR^{3\times n} \to \RR $ is of degree $\leq D$ and $\Gthree $ invariant, then we can apply Proposition \ref{prop:mace_equi} for the equivariant polynomial $p=n^{-1}(q,q,\ldots,q) $ to obtain a MACE architecture with $E_i(X)=n^{-1}q(X) $, and so $E(X)=q(X) $. 

In the proof we will use the following lemma, which shows that a product of many sums over the same points can be written using products of fewer sums.
\begin{lemma}\label{lem:newton}
Let $u_1,\ldots,u_M\in \RR^3$ and let $N\geq M$. For a function $f:\RR^3\to \RR$ denote $P(f)=\sum_{j=1}^M f(u_j)$. Then for every $f_0,\ldots,f_N:\RR^3\to \RR$, the product $P(f_0)P(f_1)\cdots P(f_N)$ is a linear combination of products of the form
$$P\Big(\prod_{s\in c_1}f_s\Big)\cdot P\Big(\prod_{s\in c_2}f_s\Big)\cdots P\Big(\prod_{s\in c_r}f_s\Big), $$
where $c_1,\ldots,c_r$ is a partition of $\{0,\ldots,N\}$ into $r\leq N$ sets. The coefficients depend only on the partition, and not on the functions $f_s$, the points $u_j$ or $M$.
\end{lemma}
\begin{proof}
We first show that
\begin{equation}\label{eq:newton}
\sum_{\sigma\in S_{N+1}}\mathrm{sgn}(\sigma)\prod_{c}P\Big(\prod_{s\in c}f_s\Big)=0,
\end{equation}
where $\sigma$ is a permutation of $\{0,\ldots,N\}$ and the product is over the cycles $c$ of $\sigma$ (including fixed points). Expanding the sums, the product over the cycles of $\sigma$ is equal to $\sum_{\psi}\prod_{s=0}^N f_s(u_{\psi(s)})$, where the sum is over all functions $\psi:\{0,\ldots,N\}\to \{1,\ldots,M\}$ which are constant on the cycles of $\sigma$, that is $\psi\circ \sigma=\psi$. Changing the order of summation, the left hand side of \eqref{eq:newton} is equal to
$$\sum_{\psi}\prod_{s=0}^N f_s(u_{\psi(s)})\sum_{\sigma: \, \psi\circ\sigma=\psi}\mathrm{sgn}(\sigma). $$
Since $N+1>M$, for every $\psi$ there exist $s\neq s'$ with $\psi(s)=\psi(s')$. Let $\tau$ be the transposition of $s$ and $s'$. Then $\sigma\mapsto \sigma\circ\tau$ is a bijection of the set $\{\sigma: \, \psi\circ\sigma=\psi\}$ to itself which changes the sign, so the inner sum is zero. This proves \eqref{eq:newton}.

The term of the identity permutation in \eqref{eq:newton} is $P(f_0)P(f_1)\cdots P(f_N)$, and every other permutation has at most $N$ cycles. Moving these terms to the other side we obtain the claim. The coefficient of a partition is minus the sum of the signs of the permutations whose cycles are the sets of the partition, so it depends only on the partition.
\end{proof}

\begin{proof}[Proof of Proposition \ref{prop:mace_equi}]
    The proof is based on the characterization of $\Gthree$ equivariant and invariant polynomials in \cite{dymMaron}.

Now fix some $i\in \{1,\ldots,n\}$ and    consider the $i$-th coordinate  $p_i:\RR^{3\times n}\to \RR$.
The permutation equivariance of $p$ implies (see \citealp[Theorem~2]{segol}) that 
$p_i $ is a linear combination of monomials of the form 
$$\sum_{i_2,i_3,\ldots,i_r=1}^n x_i^{\alpha_1}x_{i_2}^{\alpha_2}\cdots x_{i_r}^{\alpha_r} $$
where $r\geq 1$, each $\alpha_j=(\alpha_{1j},\alpha_{2j},\alpha_{3j}) $ is a multi-index and the total degree $\sum_{j=1}^r \|\alpha_j\|_1 $ is at most $D$. Since $p_i$ is translation invariant, it is unaffected by translating all entry columns by the  column $x_i$, showing that it is in fact a linear combination of expressions of the form  
\begin{equation*}\sum_{i_2,i_3,\ldots,i_r=1}^n (x_{i_2}-x_i)^{\alpha_2}\cdots (x_{i_r}-x_i)^{\alpha_r}= \left(\sum_{j\neq i} (x_j-x_i)^{\alpha_2} \right)\ldots \left(\sum_{j\neq i} (x_j-x_i)^{\alpha_r} \right) 
\end{equation*}
 Noting that factors with  $\alpha_t = 0$ are constants and are absorbed into the coefficients and every remaining factor has degree $\geq 1$, so the number of factors in the product above is at most $\sum_{t} |\alpha_t | \leq D$. Moreover, each of the sums is over the $n-1$ points $x_j-x_i$, $j\neq i$. Therefore, by applying Lemma \ref{lem:newton} repeatedly, a product of more than $n-1$ such sums is a linear combination of products of at most $n-1$ sums of the same form, where the new multi-indices are sums of the original $\alpha_t$, so the total degree does not change. The coefficients of this linear combination do not depend on $i$. Thus we can assume that the number of factors is at most $\min\{D,n-1\}$. 
  
  Next, we leverage the fact that every polynomial $q:\RR^3\to \RR $ of degree $D$ can be written as linear combination of polynomials of the form $\|x\|^{2k}Y_\ell^m(x) $ with $\ell+2k\leq D $  \citep{book_harmonic}. As a result, we deduce that every expression of the form $\left(\sum_{j\neq i} (x_j-x_i)^{\alpha} \right) $ above can be rewritten as a linear combination of polynomials of the form
  $$a_{i,k\ell m}(X)=\sum_{j\neq i}^n \|x_j-x_i\|^{2k}Y_{\ell}^m(x_j-x_i)$$
and therefore $p_i$ is a linear combination of polynomials of the form $$
c_{i,\vk,\vl,\vm}(X)=a_{i,k_1,\ell_1, m_1}(X)\cdot a_{i,k_2,\ell_2, m_2}(X) \cdot \ldots \cdot a_{i,k_\nu,\ell_\nu, m_\nu}(X) $$
where 
$$\vk=(k_1,\ldots,k_\nu), \vl=(\ell_1,\ldots,\ell_\nu), \vm=(m_1,\ldots,m_\nu) .$$
This gives that each $\sum_{j} (x_j - x_i)^{\alpha}$ is homogeneous of degree $|\alpha|$, so its expansion uses only $2k + l = | \alpha|$, therefore every product
  $c_{i,k,l,m}$ in the expansion of $p_i$ satisfies $\sum_{s} (2k_s + l_s) \leq D$. Now, let $c_{i,\vk}(X) $ be the tensor containing all entries $c_{i,\vk,\vl,\vm}(X) $, and let $$b_{i,\vk}(X)=\mathcal{C}c_{i,\vk}(X)  $$where $\mathcal C$ is the linear equivariant invertible map mapping the tensor product domain to a direct sum of irreducibles which we denote by $W=W_0^{r_0}\oplus W_1^{r_1}\oplus \ldots \oplus W_D^{r_D}$ (whose coordinates are the Clebsch-Gordan coefficients). 

Since $\sum_{s} l_s \leq D$, the irreducibles in $W$ are labeled at most by $D$. We deduce that for appropriate linear $\Lambda_{\vk}:W\to \RR $,
\begin{equation}\label{eq:pi_final}
p_i(X)=\sum_{\vk} \Lambda_{\vk}\left(b_{i,\vk}(X)\right)  \end{equation}
Moreover, if $p_i$ is additionally rotation invariant, then an averaging argument (see the proof of Theorem 1 in \citealp{dymMaron}) shows that the $\Lambda_{\vk}$ are rotation invariant. It follows almost directly from Schur's lemma (see Lemma 5 in \citealp{dymMaron} for details)  that $\Lambda_{\vk} $ only operates on the components of  $ b_{i,\vk}$ in the trivial representation $W_0^{r_0} $. Namely, if a general element in the tensor $ b_{i,\vk}$ is indexed by $b_{i,\vk,\eta,\ell,m}$, where $\eta$ ranges from $1$ to $r_\ell$, the multiplicity of $W_\ell$ in $W$, then 
\begin{equation}\label{eq:lambda}
\Lambda_{\vk}\left(b_{i,\vk} \right)=\sum_{\eta=1}^{r_0} \lambda_{\vk,\eta}b_{i,\vk\eta00} \text{ for some vector } \lambda_\vk=\left(\lambda_{\vk,1},\ldots, \lambda_{\vk,r_0}\right)
\end{equation}

 Since $\Lambda_{\vk}$ reads only the $W_0$-components of $b_{i,\vk}$, a term
$c_{i,\vk,\vl,\vm}$ contributes to $p_i$ only if $W_0$ occurs in
$W_{\ell_1}\otimes\cdots\otimes W_{\ell_\nu}$. A tensor product of irreducibles
contains no $W_J$ with $J$ larger than the sum of its labels (iterate the
Clebsch--Gordan rule $W_a\otimes W_b=\bigoplus_{c=|a-b|}^{a+b}W_c$), so if 
$W_0\subseteq W_{\ell_\xi}\otimes\bigl(\bigotimes_{\xi'\neq\xi}W_{\ell_{\xi'}}\bigr)$
then $\ell_\xi\le\sum_{\xi'\neq\xi}\ell_{\xi'}$ for every $\xi$, thus
$2\ell_\xi\le\sum_\xi\ell_\xi\le\sum_\xi(2k_\xi+\ell_\xi)\le D$. All other terms
may be dropped from the expansion of $p_i$, and we assume from now on that
$\ell_\xi\le\lfloor D/2\rfloor$ and $k_\xi\le\lfloor D/2\rfloor$ for every factor.

We now explain how MACE can express $p_i$. First note that we can express $a_{i,klm} $ using $A_{i,klm} $,  by choosing $R_{kl}(|x_j-x_i|)=|x_j-x_i|^{2k+l} $ to obtain 
$$A_{i,klm}(X)=\sum_{j\neq i} |x_j-x_i|^{2k+l}Y_l^m\left(\frac{x_j-x_i}{|x_j-x_i|}\right)\stackrel{(*)}{=}\sum_{j\neq i} |x_j-x_i|^{2k}Y_l^m\left(x_j-x_i\right)=a_{i,klm}(X),$$
where (*) follows from the fact that $Y_\ell^m $ is a degree $\ell$ homogeneous polynomial and above $A_{i,klm}$ features have $\lfloor D/2 \rfloor + 1$ channels indexed by $k = 0, …, \lfloor D/2 \rfloor$, with $R_{kl}(r)=r^{2k+l}$ for $l\leq \lfloor D/2 \rfloor$.

Next, we choose 
$$w_{\vk \tilde k \xi }=\twopartdef{1}{\tilde k=k_\xi}{0}{{\tilde k\neq k_\xi}}, \text{ so that } \sum_{\tilde k} w_{k \tilde k \xi}\, A_{i,\tilde k l_\xi m_\xi}= A_{i,k_\xi l_\xi m_\xi}$$
and we can then obtain that 
\begin{align*}
B_{i,\vk}&=\sum_{\boldsymbol{lm}}
    \mathcal C^{LM}_{\eta_\nu, \boldsymbol{lm}}
    \prod_{\xi=1}^{\nu} \sum_{\tilde k} w_{k \tilde k \xi}\, A_{i,\tilde k l_\xi m_\xi}\\
    &=
    \sum_{\boldsymbol{lm}}
    \mathcal C^{LM}_{\eta_\nu, \boldsymbol{lm}}
    \prod_{\xi=1}^{\nu} \, A_{i,k_\xi l_\xi m_\xi}\\
    &=
\mathcal{C}\left(a_{i,k_1}\otimes a_{i,k_2} \otimes \ldots \otimes a_{i,k_\nu} \right) 
\end{align*}
where $\mathcal{C}$ is the linear mapping defined by the Clebsch-Gordan coefficients  $\mathcal C^{LM}_{\eta_\nu, \boldsymbol{lm}} $. The $B_{i, \vk}$ are indexed by the tuple $\vk = (k_1, \ldots, k_{\nu})$ and there are $\Kmax(D)$ valid tuples, that is those with $k_{i} \leq \lfloor D/2 \rfloor$ and $\nu \leq \min\{D,n-1\}$, which are those we iterate over so that $w_{k\tilde k \xi}$ has a well-defined index range. The $L=0$ components of $B_{i,\vk}$ coincide with those of $b_{i,\vk}$ used in \Cref{eq:lambda}, a fact we use below. 

Thus, choosing $W_{\vk 0 \eta}=\lambda_{\vk \eta} $ and $W_{\vk L \eta}=0$ for $L>0$, and choosing the MLP in the definition of $E_i$ to just sum over all coordinates $\vk$, we obtain 
\begin{align*}
p_i(X)&\stackrel{\eqref{eq:pi_final}}{=} \sum_{\vk} \Lambda_{\vk}\left(b_{i,\vk}(X)\right)\stackrel{\eqref{eq:lambda}}{=}
\sum_{\vk}\sum_{\eta=1}^{r_0} \lambda_{\vk,\eta}b_{i,\vk\eta00}=E_i(X) 
\end{align*}

\end{proof}

\subsection{MACE with multiple iterations}\label{app:MACE_multi}
So far we considered MACE with a single iteration. We now explain how the results above extend to MACE with $T\geq 1$ iterations. As above, we perform message passing on all points, so there is no underlying graph. In the MACE architecture \citep{mace} with $T$ iterations, every point $x_i$ has a hidden state $h^{(t)}_i$, where $h^{(0)}_i=1$, and the hidden states are updated by
\begin{align}
A^{(t)}_{i,k\ell_3 m_3} &= \sum_{\ell_1 m_1,\ell_2 m_2}\mathcal C^{\ell_3 m_3}_{\ell_1 m_1,\ell_2 m_2}\sum_{j\neq i} R^{(t)}_{k\ell_1\ell_2\ell_3}(|x_j-x_i|)\, Y^{m_1}_{\ell_1}\left(\frac{x_j-x_i}{|x_j-x_i|}\right)\sum_{\tilde k} W^{(t)}_{k\tilde k\ell_2}\, h^{(t)}_{j,\tilde k\ell_2 m_2} \label{eq:At}\\
h^{(t+1)}_{i,kLM} &= \sum_{\tilde k} U^{(t)}_{kL,\tilde k}\, m^{(t)}_{i,\tilde k LM}+\sum_{\tilde k} \tilde U^{(t)}_{kL,\tilde k}\, h^{(t)}_{i,\tilde k LM} \label{eq:update}
\end{align}
for $t=0,\ldots,T-1$, where $B^{(t)}_i$ and $m^{(t)}_i$ are computed from $A^{(t)}_i$ as in \eqref{eq:B} and \eqref{eq:pooling}, and the node feature $E_i$ is obtained by applying the MLP to the features $h^{(T)}_{i,k00}$. Note that for $t=0$ we have $h^{(0)}_j=1$, so only $\ell_2=0$ appears and \eqref{eq:At} is the same as \eqref{eq:A}. Therefore the single iteration model above is obtained by choosing $T=1$, $U^{(0)}$ to be the identity and $\tilde U^{(0)}=0$.

\begin{proposition}\label{prop:mace_multi}
Let $T\geq 1$. For every MACE model with a single iteration, there exists a MACE model with $T$ iterations and the same $\numax$, $\ellmax$ and $\Kmax$, which has the same node features $E_i$. In particular, Proposition \ref{prop:mace} and Proposition \ref{prop:mace_equi} hold for MACE with $T$ iterations. The same is true when the summation in \eqref{eq:A} and \eqref{eq:At} is only over the points at distance at most $\rho$ from $x_i$.
\end{proposition}
\begin{proof}
In the first iteration we choose the parameters of the single iteration model, $U^{(0)}$ to be the identity and $\tilde U^{(0)}=0$, so that $h^{(1)}_i=m^{(0)}_i$. For $t\geq 1$ we choose $U^{(t)}=0$ and $\tilde U^{(t)}$ to be the identity, so that $h^{(t+1)}_i=h^{(t)}_i$. We obtain that $h^{(T)}_i=h^{(1)}_i$, and therefore the node features $E_i$ are the same as in the single iteration model. The same argument applies when the summation in \eqref{eq:At} is only over the points at distance at most $\rho$ from $x_i$.
\end{proof}

In the proof of Proposition \ref{prop:mace_multi} the additional iterations are not used. We now show that also when all iterations are used, features with $\ell>\lfloor D/2 \rfloor$ do not contribute to polynomials of degree $\leq D$. As in the proof of Proposition \ref{prop:mace_equi}, we consider radial functions of the form $R^{(t)}_{k\ell_1\ell_2\ell_3}(r)=c\, r^{2a+\ell_1}$ where $a\geq 0$ is an integer, so that $R(|x|)Y^{m_1}_{\ell_1}(x/|x|)=c|x|^{2a}Y^{m_1}_{\ell_1}(x)$ is a homogeneous polynomial of degree $2a+\ell_1$. We also assume that in each of \eqref{eq:At}, \eqref{eq:B}, \eqref{eq:pooling} and \eqref{eq:update}, all the terms in the sum which defines a given coordinate are homogeneous polynomials in $X$ of the same degree. A model whose radial functions are linear combinations of the functions $r^{2a+\ell_1}$, and in particular the model in the proof of Proposition \ref{prop:mace_equi}, can be brought to this form without changing $T$ and $\numax$, by splitting every channel according to the degrees, since the radial functions can be chosen separately for every channel and the weights can select a single term in each of these sums. We say that a coordinate of $A^{(t)}_i, B^{(t)}_i,m^{(t)}_i$ or $h^{(t)}_i$ is of type $(e,L)$ if it is a homogeneous polynomial of degree $e$ and it is a coordinate of a feature in $W_L$.

\begin{lemma}\label{lem:budget}
Consider a MACE model with $T\geq 1$ iterations as above, with no restriction on $\ell$, where $E_i$ is a linear combination of the coordinates $h^{(T)}_{i,k00}$. Let $E_i^{\leq D}$ be the sum of the terms in this linear combination whose degree is at most $D$. Then $E_i^{\leq D}$ does not change if we discard the spherical harmonics with $\ell_1>\lfloor D/2\rfloor$ from \eqref{eq:At}, and the coordinates with $\ell>\lfloor D/2 \rfloor$ from all the features $A^{(t)}_i,B^{(t)}_i,m^{(t)}_i$ and $h^{(t)}_i$. In particular, if $E_i$ is a polynomial of degree $\leq D$ then it can be expressed with $\ellmax=\lfloor D/2 \rfloor$.
\end{lemma}
\begin{proof}
We first note that every coordinate of type $(e,L)$ satisfies $L\leq e$. This holds for $h^{(0)}_i=1$ which is of type $(0,0)$. In \eqref{eq:At}, a term which is obtained from a spherical harmonic of degree $\ell_1$, a radial function $c\,r^{2a+\ell_1}$ and a coordinate of $h^{(t)}_j$ of type $(e_2,\ell_2)$ with $\ell_2\leq e_2$, is of type $(e_3,\ell_3)$ with $e_3=2a+\ell_1+e_2\geq \ell_1+\ell_2\geq \ell_3$. In \eqref{eq:B}, a product of coordinates of types $(e_\xi,\ell_\xi)$ with $\ell_\xi\leq e_\xi$ is of type $(e,L)$ with $e=\sum_\xi e_\xi\geq \sum_\xi \ell_\xi\geq L$. Finally, \eqref{eq:pooling} and \eqref{eq:update} do not change the type.

We say that a coordinate of type $(e,L)$ is \emph{useful} if $e+L\leq D$. A useful coordinate satisfies $2L\leq e+L\leq D$ and therefore $L\leq \lfloor D/2 \rfloor$. We claim that a useful coordinate is computed only from useful coordinates, and from spherical harmonics with $\ell_1\leq \lfloor D/2\rfloor$. For \eqref{eq:pooling} and \eqref{eq:update} this is clear. For \eqref{eq:B}, if $W_L$ occurs in $W_{\ell_1}\otimes\cdots\otimes W_{\ell_\nu}$ then by the Clebsch--Gordan rule $\ell_\xi\leq L+\sum_{\xi'\neq\xi}\ell_{\xi'}$ for every $\xi$, and therefore
$$e_\xi+\ell_\xi\leq e_\xi+L+\sum_{\xi'\neq \xi}\ell_{\xi'}\leq L+\sum_{\xi'}e_{\xi'}=e+L\leq D. $$
For \eqref{eq:At} we have $\ell_2\leq \ell_1+\ell_3$ and $\ell_1\leq \ell_2+\ell_3$, and therefore
$$e_2+\ell_2\leq e_2+\ell_1+\ell_3\leq e_3+\ell_3\leq D \quad \text{and} \quad 2\ell_1\leq \ell_1+\ell_2+\ell_3\leq e_3+\ell_3\leq D. $$
It follows by induction on $t$ that the useful coordinates do not change when we discard the coordinates with $\ell>\lfloor D/2 \rfloor$, since these coordinates are not useful. The coordinates $h^{(T)}_{i,k00}$ of degree $\leq D$ are of type $(e,0)$ with $e\leq D$, so they are useful, and thus $E_i^{\leq D}$ does not change. Finally, if $E_i$ is a polynomial of degree $\leq D$ then the terms of degree larger than $D$ cancel out, so $E_i=E_i^{\leq D}$, and we can choose the coefficients of these terms to be zero.
\end{proof}

We now show that when the summation is over all points, more iterations can be used instead of a higher $\numax$. For $\nu\geq 1$ we define $\mu_0=0$ and $\mu_t=\nu(\mu_{t-1}+1)$ for $t\geq 1$, that is $\mu_t=\nu+\nu^2+\ldots+\nu^t$. In particular $\mu_1=\nu$, for $\nu=1$ we have $\mu_t=t$, and for $\nu=2$ we have $\mu_t=2^{t+1}-2$.

\begin{proposition}\label{prop:mace_tradeoff}
Let $n\geq 2$ and let $\nu,T\geq 1$ be such that $\mu_T\geq \min\{D,n-1\}$. Then Proposition \ref{prop:mace} and Proposition \ref{prop:mace_equi} hold also for MACE with $T$ iterations, $\numax=\nu$ and $\ellmax=\lfloor D/2\rfloor$, when the summation in \eqref{eq:At} is over all points.
\end{proposition}

As a corollary of this proposition and Theorem \ref{thm:dn2}, we obtain the following variant of Corollary \ref{cor:maceLocal}:
\begin{corollary}\label{cor:mace_multi}
Let $n\geq 2$ and set $D=18n^2-18n+9$. Let $\nu,T\geq 1$ be such that $\mu_T\geq n-1$. Then a MACE model with $T$ iterations, $\numax=\nu$ and $\ellmax=\lfloor D/2\rfloor$, in which the summation in \eqref{eq:At} is over all points, is $(\Gthree,\RR^{3\times n})$ complete. In particular, this holds for $\numax=n-1$ and $T=1$, for $\numax=2$ and $T\geq \log_2(n+1)-1$, and for $\numax=1$ and $T\geq n-1$.
\end{corollary}

\begin{proof}[Proof of Proposition \ref{prop:mace_tradeoff}]
For $m\geq 0$ and a polynomial $H:(\RR^3)^m\to W_L$ of degree $\leq D$ we denote
$$S_i(H)=\sum_{j_1,\ldots,j_m=1}^n H(x_{j_1}-x_i,\ldots,x_{j_m}-x_i), $$
and if $H$ is equivariant we say that $S_i(H)$ is an $m$-point sum centered at $x_i$ (for $m=0$ these are the constants in $W_0$). Note that an $m$-point sum is also an $m'$-point sum for every $m'\geq m$, since if $H$ does not depend on the last $m'-m$ points then the sum over these points gives the constant $n^{m'-m}$. We prove by induction on $t$ that we can choose the parameters of the first $t$ iterations such that for every $i$, every $\mu_t$-point sum centered at $x_i$ is a linear combination of the features $h^{(t)}_{i,kLM}$, with coefficients which do not depend on $i$. For $t=0$ this holds since $h^{(0)}_i=1$.

Now assume that this holds for $t$, and denote $m=\mu_t$. Fix $i$ and denote $u_j=x_j-x_i$, so that $u_i=0$. Let $\phi:\RR^3\to W_{\ell_1}$ and $H:(\RR^3)^m\to W_{\ell_2}$ be equivariant polynomials. Then $\phi(x)=R(|x|)Y_{\ell_1}(x/|x|)$, where $Y_{\ell_1}$ is the vector of the spherical harmonics of degree $\ell_1$ and $R$ is a linear combination of the functions $r^{2a+\ell_1}$, and by the induction hypothesis $S_j(H)$ is a linear combination of the features $h^{(t)}_j$. Therefore in \eqref{eq:At} we can obtain the features
$$\sum_{j\neq i}\phi(u_j)\otimes S_j(H)=\sum_{j\neq i}\, \sum_{j_1,\ldots,j_m=1}^n \phi(u_j)\otimes H(u_{j_1}-u_j,\ldots,u_{j_m}-u_j)=S_i(\tilde H)-S_i(\tilde H_0), $$
where
$$\tilde H(u,w_1,\ldots,w_m)=\phi(u)\otimes H(w_1-u,\ldots,w_m-u), \quad \tilde H_0(w_1,\ldots,w_m)=\tilde H(0,w_1,\ldots,w_m), $$
and the tensor product is decomposed into the components in $W_{\ell_3}$ using the Clebsch-Gordan coefficients. Here we used $x_{j_s}-x_j=u_{j_s}-u_j$, and the fact that the term of $j=i$ in $S_i(\tilde H)$ is $S_i(\tilde H_0)$. Recall that every equivariant polynomial of degree $\leq D$ in two groups of points is a linear combination of products of equivariant polynomials in each of the groups, decomposed using the Clebsch-Gordan coefficients, since the polynomials in both groups are the tensor product of the polynomials in each group. The same holds for $\nu$ groups of points, with the generalized Clebsch-Gordan coefficients of \eqref{eq:B}. Since $(u,w_1,\ldots,w_m)\mapsto (u,w_1-u,\ldots,w_m-u)$ is an invertible linear map which commutes with the rotations, it follows that every equivariant polynomial $G:(\RR^3)^{m+1}\to W_L$ of degree $\leq D$ is a linear combination of polynomials of the form $\tilde H$, and thus we obtain $S_i(G)-S_i(G_0)$ for every such $G$, where $G_0(w_1,\ldots,w_m)=G(0,w_1,\ldots,w_m)$. If $G$ does not depend on $u$ then $S_i(G)=nS_i(G_0)$, so we obtain $(n-1)S_i(G_0)$, and since $n\geq 2$ we obtain all $m$-point sums. It follows that every $(m+1)$-point sum is a linear combination of the features $A^{(t)}_i$. In \eqref{eq:B} we obtain products of $\nu$ such sums, which are the $\nu(m+1)$-point sums $S_i(H_1\otimes\cdots\otimes H_\nu)$, where $H_\xi$ depends on the $\xi$-th group of $m+1$ points. Every equivariant polynomial of degree $\leq D$ in $\nu(m+1)$ points is a linear combination of such products, and therefore every $\mu_{t+1}$-point sum is a linear combination of the features $h^{(t+1)}_i$.

Finally, let $p$ be as in Proposition \ref{prop:mace_equi}. By the proof of that proposition, $p_i$ is a linear combination of products of at most $\min\{D,{n-1}\}\leq \mu_T$ sums of the form $\sum_{j\neq i}(x_j-x_i)^{\alpha}$ with $\alpha\neq 0$, with coefficients which do not depend on $i$. The term of $j=i$ is zero, so these sums are equal to $\sum_{j=1}^n (x_j-x_i)^{\alpha}$. Therefore $p_i=S_i(H)$ for a polynomial $H:(\RR^3)^{\mu_T}\to \RR$ of degree $\leq D$ which does not depend on $i$, and since $p_i$ is rotation invariant we can choose $H$ to be rotation invariant by averaging over the rotations. Thus $p_i$ is a $\mu_T$-point sum, so it is a linear combination of the features $h^{(T)}_{i,k00}$, and we choose the MLP to be this linear combination. In this construction all the polynomials can be chosen to be homogeneous, with a separate channel for every polynomial (if $\phi$ and $H$ are homogeneous then $\tilde H$ is homogeneous of degree $\deg \phi+\deg H$, and $\tilde H_0\neq 0$ only if $\phi$ is constant), so the assumptions of Lemma \ref{lem:budget} hold, and therefore $\ellmax=\lfloor D/2\rfloor$ is sufficient. Proposition \ref{prop:mace} follows as in the case of a single iteration.
\end{proof}

We now give an example which illustrates Proposition \ref{prop:mace_tradeoff} for $\nu=2$ and $T=2$. For $n\geq 4$, the polynomial in this example cannot be expressed by a single iteration with $\numax=2$ and a linear readout as in Lemma \ref{lem:budget}.

\begin{example}\label{ex:two_iterations}
Fix $i$ and denote $u_j=x_j-x_i$. Let $v_i=\sum_{j\neq i}u_j$ and $M_i=\sum_{j\neq i}u_ju_j^T$, and consider the invariant polynomial
$$q_i(X)=v_i^TM_iv_i=\sum_{j_1,j_2,j_3\neq i}\langle u_{j_1},u_{j_2}\rangle \langle u_{j_2},u_{j_3}\rangle .$$
A single iteration of MACE with $\numax=3$ can express $q_i$, since it is a linear combination of products of three of the features $a_{i,k\ell m}$. For $n\geq 4$, a single iteration with $\numax=2$ in which $E_i$ is a linear combination of the features $m_{i,k00}$ cannot express $q_i$: these features are sums of functions which depend on at most two of the points $u_j$, while $q_i$ contains the monomials of $\langle u_{j_1},u_{j_2}\rangle \langle u_{j_2},u_{j_3}\rangle$ with $j_1,j_2,j_3$ distinct.

We now show that two iterations with $\numax=2$ and $\ellmax=2$ can express $q_i$. In the first iteration we compute for every point $x_j$ the feature $v_j=\sum_{j'\neq j}(x_{j'}-x_j)$, using $\nu=1$, $\ell=1$ and $R(r)=r$, and we keep the constant feature $h^{(0)}_j=1$ using $\tilde U^{(0)}$. Since the summation is over all points, we have
$$v_j=\sum_{j'=1}^n (x_{j'}-x_i)-n(x_j-x_i)=v_i-nu_j .$$
Therefore in the second iteration, choosing $\ell_1=\ell_2=1$ and $R(r)=r$ in \eqref{eq:At}, we obtain the entries of
$$\sum_{j\neq i}u_j\otimes v_j=v_i\otimes v_i-nM_i ,$$
and choosing $\ell_2=0$, $\ell_1\in\{0,2\}$ and $R(r)=r^2$ we obtain the entries of $M_i$. Using $\nu=2$ we obtain
$$\langle M_i,v_i\otimes v_i-nM_i\rangle +n\langle M_i,M_i\rangle =v_i^TM_iv_i=q_i(X).$$
\end{example}

\begin{remark}\label{rem:not_identical}
The proof of Proposition \ref{prop:mace_tradeoff} uses the fact that the summation is over all points. For $t\geq 1$ the features $A^{(t)}_i$ in \eqref{eq:At} are computed from the hidden states $h^{(t)}_j$ of the other points, while $h^{(t)}_i$ itself only appears in the linear term of \eqref{eq:update}, so the sums from the previous iterations are centered at the points $x_j$. In the proof we rewrote them as sums centered at $x_i$. When the summation in \eqref{eq:At} is only over the points at distance at most $\rho$ from $x_i$, in general the neighbors of $x_j$ are different from the neighbors of $x_i$, and this argument does not apply.
\end{remark}

\section{Local models: proof of Theorem~\ref{thm:maceLocal}}\label{app:local}
Here we prove Theorem \ref{thm:maceLocal}. We begin by restating the theorem. 
\maceLocal*

\subsection{Proof of Theorem~\ref{thm:maceLocal}}\label{app:local:first}
The equations for a single MACE iteration are those given above, with the only difference that
the summation in \eqref{eq:A} is over the points at distance at most $\rho$ from $x_i$ rather
than over all points. In this subsection we call this version MACE, or local MACE, and refer to
the option of summing over all points as global MACE. The theorem rests on the
following lemma.

\begin{lemma}\label{lem:local}
There exists an instantiation of MACE with $T=1$, $\numax=K-1$ and $\ellmax=\lfloor D/2 \rfloor$, where $D=D(K)=18K^2-18K+9$, such that, if $X,Y$ are $(\rho,K) $ regular, then
\begin{enumerate}
\item For all $1\leq i,k\leq n$, if $E_i(X)=E_k(Y)$ then there exists a rotation $R\in SO(d) $, such that 
$$\{x_j-x_i| \quad \|x_j-x_i\|\leq \rho \}=\{R(y_j-y_k)| \quad \|y_j-y_k\| \leq \rho \}. $$
\item If $E(X)=E(Y)$ then the node features of $X$ and $Y$ are equal, up to permutation, namely we have the multiset equality 
$$\{E_i(X), i=1,\ldots,n\}=\{E_i(Y), i=1,\ldots,n\} $$
\end{enumerate}
\end{lemma}
\begin{proof}[Proof of Lemma \ref{lem:local}]
To prove the first part of the lemma, we note that if we remove from $X$ and $Y$ the points which are at distance more than $\rho$ from $x_i$ and $y_k$, this does not influence the features $E_i(X)$ and $E_k(Y)$ obtained from the first iteration of MACE. Moreover, the feature $A_{i,000}$ with $R_{00}=1$ counts the number of points, so point clouds with a different number of points are separated. Accordingly, it is sufficient to consider point clouds $X,Y$ which have the same number of points $n\leq K$ (and $n\geq d+1$ by \textnormal{(O)} together with the overlap requirement in the definition of $\rho$-adjacency), all of them at distance at most $\rho$ from $x_i$ and $y_k$ respectively. To prove this, it is sufficient to show that we can choose an instantiation of global MACE such that,  if there exists some  point clouds $X,Y\in \RR^{d\times n}$ and $i,k\in [n] $ such that 
$$\{x_j-x_i| \quad j\in [n] \}\neq \{R(y_j-y_k)| \quad j\in [n] \}, \quad \forall R\in SO(d), $$
then $E_i(X)\neq E_k(Y) $. This follows from our results on global MACE and invariant polynomials. 

Indeed, by Theorem \ref{thm:dn2} there exists an $SO(d)\times S_{n}$ invariant polynomial $\phi:\RR^{d\times n}\to \RR$ of degree $\leq D(n)\leq D(K) $  such that 
$$\phi(x_1-x_i,\ldots,x_n-x_i)\neq \phi(y_1-y_k,\ldots,y_n-y_k) . $$
Using this, we can define a permutation equivariant and rotation and translation invariant polynomial $p:\RR^{d\times n}\to \RR^n$ whose coordinates are given by 
$$p_i(Z)=\phi(z_1-z_i,\ldots,z_n-z_i) $$
By Proposition \ref{prop:mace_equi},  there exists a global MACE model with the prescribed size, such that $p(Z)=(E_1(Z),\ldots,E_n(Z)) $, and then 
$$E_i(X)=\phi(x_1-x_i,\ldots,x_n-x_i)\neq \phi(y_1-y_k,\ldots,y_n-y_k)=E_k(Y)$$
This concludes the proof of the first part of the claim (note that we in fact proved pairwise separation by a MACE architecture, but since the polynomials of bounded degree have a finite space, the extension to global separation is immediate. The same architecture can be used for all $n\leq K$: we choose the node feature $E_i$ to be the vector of all the invariant features $B_{i,\eta_\nu \vk 00}$ with $\nu\leq K-1$ from the proof of Proposition \ref{prop:mace_equi}, which include the feature $A_{i,000}$ above, using a separate channel in \eqref{eq:pooling} for every $\eta_\nu$. Then by \eqref{eq:pi_final} and \eqref{eq:lambda} all the polynomials $p_i$ above are linear combinations of the coordinates of $E_i$).

The second part of the claim follows from the fact that there exist MLPs such that the mapping $(E_1,\ldots,E_n)\mapsto \sum_{i=1}^n \mathrm{MLP}(E_i) $ is injective as a function of multisets (see \citealp{amir2023neural}). 
\end{proof}

\begin{proof}[Proof of Theorem~\ref{thm:maceLocal}]
We choose the MACE architecture that satisfies the conditions of Lemma~\ref{lem:local}. For $T>1$ iterations, by Proposition \ref{prop:mace_multi} there exists a MACE architecture with $T$ iterations which has the same node features, so it is sufficient to consider $T=1$. It remains to show that these conditions imply completeness on the subset of strongly regular point clouds. To do this, it is sufficient to show that, if we are given a multiset of node features 
$$\{E_1,\ldots,E_n\}$$
obtained from applying MACE to some strongly regular $X$, then we can uniquely reconstruct $X$, up to an $SO(d)\times S_n $ transformation. The lemma immediately allows us to reconstruct from each feature $E_i$, a subset $X[i]$ of $\leq K$ point of $X$, up to rigid motion. What remains to show is that there is a unique reconstruction up to a single rigid motion.

We perform this reconstruction as follows: at a first step, we choose some arbitrary node feature and name it $E_1$. We choose some arbitrary orientation for $X[1]$, the subset inferred from $E_1$. It now remains to show that the remainder of $X$ is uniquely defined by this choice. This is done iteratively:

First, we go over all $E_j$: from each $E_j$ we can reconstruct $X[j]$ up to rigid motion. We can then check whether $X[j]$  has a subset of cardinality $d+1$ which is congruent to a subset of $d+1$ points in $X[1]$, which is a finite combinatorial search problem. If so, then by strong regularity these congruent points are in fact identical, and so we know their position. Once these points are fixed, we can deduce where all of the $X[j]$ points are, since by assumption these $d+1$ points are affine independent, and we know $X[j]$ up to rigid motion. By the overlap requirement in the definition of $\rho$-adjacency, the reconstruction step will succeed for all $j$ which are adjacent to $1$. 

In the next step, we repeat this, now comparing each not yet reconstructed set $X[j]$ with all the sets which were already reconstructed. We continue iteratively.  Since we assume that the $\rho$-adjacency graph is connected, this procedure will terminate with the full reconstruction of $X$. 
\end{proof}

\subsection{Lower bounds on hyperparameters}\label{app:local:lower}\label{app:local:lower:rings}

Throughout this subsection a \emph{local MACE model} with cutoff $\rho$, depth $T$ and hyperparameters $\ellmax,\numax,\nuread$ is the model \eqref{eq:A}--\eqref{eq:E} iterated $T$ times, with the sum in \eqref{eq:A}  restricted to $x_j\in B(x,\rho)\cap X$ for $x=x_i$ and the neighbour's feature of the previous iteration inserted in it: the feature vector of any $x\in X$ at step $t$ is denoted $h^{(t)}(x)$ and lies in a finite sum of copies of $W_0\oplus\dots\oplus W_{\ellmax}$, one iteration is a polynomial of degree at most $\numax$ in the aggregated features $A$, the site output is an invariant polynomial of degree at most $\nuread$ in $h^{(T)}(x)$, and the cloud output is an arbitrary function of the multiset of site outputs. The radial functions $R_{k\ell}$ are arbitrary
continuous functions, the number of channels is unrestricted, and arbitrary continuous functions
may be applied to invariant channels at any stage, so that the MLPs of \eqref{eq:Ei} are covered.

The following family is the local form of the counter-rotating polygons used to prove
Theorem~\ref{thm:lower-quadratic}. Its whole point is that all the distances a model can read are
constant along it, so that the radial functions carry no information whatsoever and only the
angular resolution of the model is at stake.

\begin{definition}\label{lg:lb:def:rings}
Let $p,q\ge3$. For $\alpha,\beta\in\RR$ we define $Y(\alpha,\beta)\subset\RR^3$ to be the set containing the origin $c$ together with
\[
a_j=\Bigl(\tfrac{9}{20}\cos\bigl(\tfrac{2\pi j}{p}+\alpha\bigr),\
\tfrac{9}{20}\sin\bigl(\tfrac{2\pi j}{p}+\alpha\bigr),\ -\tfrac35\Bigr),\quad
b_k=\Bigl(\tfrac{7}{20}\cos\bigl(\tfrac{2\pi k}{q}+\beta\bigr),\
\tfrac{7}{20}\sin\bigl(\tfrac{2\pi k}{q}+\beta\bigr),\ \tfrac35\Bigr),
\]
for $j\le p$ and $k\le q$, and let $\mathcal Y_{p,q}:=\{Y(\alpha,\beta):\ \alpha,\beta\in \RR\}$, a family of clouds of
$n=p+q+1$ points.
\end{definition}

\begin{lemma}\label{lg:lb:lem:rings}
Every $Y\in\mathcal Y_{p,q}$ is $(1,n)$-regular. Its cutoff balls are
$\mathcal B^Y_1(c)=Y$, $\mathcal B^Y_1(a_j)=\{c\}\cup\{a_1,\dots,a_p\}$ and
$\mathcal B^Y_1(b_k)=\{c\}\cup\{b_1,\dots,b_q\}$. Moreover $Y(\alpha,\beta)$ and
$Y(\alpha',\beta')$ lie in the same $\Gthree$-orbit if and only if
$\alpha-\beta\equiv\alpha'-\beta'\pmod{2\pi/\mathrm{lcm}(p,q)}$, so $\mathcal Y_{p,q}$ meets infinitely many
orbits.
\end{lemma}

\begin{proof}
We note that for all choices of indices $j, j', k, k'$ there holds $\|c-a_j\|=\sqrt{(9/20)^2+(3/5)^2}=3/4$, $\|c-b_k\|=\sqrt{(7/20)^2+(3/5)^2}<3/4$,
$\|a_j-a_{j'}\|\le 9/10$, $\|b_k-b_{k'}\|\le7/10$ and $\|a_j-b_k\|\ge6/5$. This gives the three
cutoff balls $\mathcal B^Y_1(\cdot)$ as claimed, and (D), (K) with $K=n$ are immediate. For (O), the ball of $a_j$ meets the ball of
$c$ in $\{c,a_1,\dots,a_p\}$, which contains three non-collinear points of the plane
$z=-3/5$ together with $c$, hence four affinely independent points, so $a_j$ and $c$ are
$\rho$-adjacent; the same holds for each $b_k$, and the adjacency graph is connected.

For the second claim, $c$ is the only point of $Y$ at distance $3/4$ from $p$ points of a common
plane and $<3/4$ from $q$ points of another, so any isometry carrying $Y(\alpha,\beta)$ to
$Y(\alpha',\beta')$ fixes $c$ and, the two rings having different cardinalities, radii and heights,
maps each ring to the corresponding one. It therefore preserves the axis and each ring plane and,
being orientation preserving, is the rotation by some angle $s$ about the axis, denoted $g_s$. Since the
ring phases are defined modulo $2\pi/p$ and $2\pi/q$, this says
$\alpha'-\alpha\in s+\frac{2\pi}{p}\ZZ$ and $\beta'-\beta\in s+\frac{2\pi}{q}\ZZ$, whence
$(\alpha'-\beta')-(\alpha-\beta)\in\frac{2\pi}{p}\ZZ+\frac{2\pi}{q}\ZZ=\frac{2\pi}{\mathrm{lcm}(p,q)}\ZZ$. The converse is clear.
\end{proof}

\begin{lemma}[frequency budget]\label{lg:lb:lem:budget}
Let a local MACE model be applied to elements of $\mathcal Y_{p,q}$, and suppose that the continuous functions it applies to invariant channels are applied only to channels that are constant on $\mathcal Y_{p,q}$; by the proof of Theorem~\ref{lg:lb:thm:main}, this holds automatically whenever $\ellmax\sum_{i=1}^{T}\numax^{\,i}<2\min(p,q)$. Then every component of every
$h^{(t)}(x)$ is a trigonometric polynomial in $(\alpha,\beta)$ with Fourier support contained in
$\{(m_1,m_2)\in\ZZ^2:|m_1|+|m_2|\le\omega_t\}$, where $\omega_0=0$ and
$\omega_t=\numax(\ellmax+\omega_{t-1})$, so that
$\omega_T=\ellmax\sum_{i=1}^{T}\numax^{\,i}$; and the site outputs have Fourier support in
$|m_1|+|m_2|\le\nuread\,\omega_T$.
\end{lemma}

\begin{proof}
By Lemma~\ref{lg:lb:lem:rings} every pair of points within the cutoff consists of $c$ and a ring
point or of two points of the same ring. The length of such a pair does not depend on
$(\alpha,\beta)$, so every $R_{k\ell}$ in \eqref{eq:A} contributes a constant. Its direction is
$g_\alpha\hat u$ for a fixed $\hat u$ (and where $g_s$ denotes the rotation about the axis of $Y(\alpha,\beta)$ by angle $s$) when both points belong to ring $A$ or to $\{c\}\cup$ring
$A$, because increasing $\alpha$ by $s$ rotates that part of the cloud rigidly by $g_s$; and it is
$g_\beta\hat u$ in the two symmetric cases. Since
$Y^{m}_{\ell}(g_\alpha\hat u)=e^{\mathrm{i}m\alpha}Y^{m}_{\ell}(\hat u)$ for the rotation about the
axis, each such factor has Fourier support $|m_1|\le\ell\le\ellmax, m_2=0$ or $|m_2|\le\ell\le\ellmax, m_1=0$. Hence each summand of \eqref{eq:A}, being a product of a constant, such a factor and a
feature of the previous iteration, has support bounded by $\ellmax+\omega_{t-1}$, and the same holds for the total sum of these terms. Passing to \eqref{eq:B} multiplies at most $\numax$ such objects and supports add under
multiplication, giving $\omega_t\le\numax(\ellmax+\omega_{t-1})$; \eqref{eq:pooling} is linear, and by assumption the continuous functions applied to invariant channels only produce constants.
The site output is an invariant polynomial of degree at most $\nuread$ in $h^{(T)}(x)$, so the
same argument bounds its support by $\nuread\omega_T$.
\end{proof}

The site outputs are constrained by the symmetries of the family in a way that depends on the
site. Write $s_x(\alpha,\beta)$ for the output at the site $x$. Increasing $\beta$ by $2\pi/q$
leaves the cloud unchanged and fixes $c$ and every $a_j$, so $s_c$ and $s_{a_j}$ are
$2\pi/q$-periodic in $\beta$; increasing $\alpha$ by $2\pi/p$ leaves the cloud unchanged and
fixes $c$ but carries $a_j$ to $a_{j+1}$, so $s_c$ is $2\pi/p$-periodic in $\alpha$ whereas
$s_{a_j}(\alpha+2\pi/p,\beta)=s_{a_{j+1}}(\alpha,\beta)$: the multiset of ring outputs is
periodic, an individual ring output is not. Finally every site output is unchanged by the
rotation $g_s$, which carries $Y(\alpha,\beta)$ to $Y(\alpha+s,\beta+s)$.

\begin{theorem}\label{lg:lb:thm:main}
Let $p,q\ge3$. If a local MACE model separates $\mathcal Y_{p,q}$ then
\[
\nuread\cdot\ellmax\cdot\sum_{i=1}^{T}\numax^{\,i}\ \ge\ 2\min(p,q),
\]
and if its readout is that of \eqref{eq:Ei}--\eqref{eq:E}, an arbitrary function of the
invariant channels at each point followed by a sum, then already
$\ellmax\sum_{i=1}^{T}\numax^{\,i}\ge2\min(p,q)$. Consequently, on $\mathcal R_{\rho,K}$ with
$K\ge7$, completeness requires $\nuread\ellmax\sum_{i\le T}\numax^i\ge2\lfloor(K-1)/2\rfloor$,
and $\ellmax\sum_{i\le T}\numax^i\ge2\lfloor(K-1)/2\rfloor$ for the readout of \eqref{eq:Ei}.
\end{theorem}

\begin{proof}
Both bounds hold if $\omega_T=\ellmax\sum_{i=1}^{T}\numax^{\,i}\geq2\min(p,q)$, since $\nuread\geq1$, so we may assume that $\omega_T<2\min(p,q)$. We first check that Lemma~\ref{lg:lb:lem:budget} then applies. An invariant channel of $h^{(t)}(x)$ is invariant under the rotation $g_s$ and under the shifts of $\alpha$ and $\beta$ that fix $x$, exactly as the site outputs above, so, by the argument below, it is constant unless one of its modes has $|m_1|+|m_2|\geq2\min(p,q)$. By induction on $t\leq T$, the continuous functions applied in the first $t$ iterations act on constant channels, so the invariant channels of $h^{(t)}(x)$ have support in $|m_1|+|m_2|\leq\omega_t\leq\omega_T<2\min(p,q)$ and are therefore constant as well. Hence the hypothesis of Lemma~\ref{lg:lb:lem:budget} holds.

If every site output were constant on $\mathcal Y_{p,q}$ the cloud output would be constant
too, against Lemma~\ref{lg:lb:lem:rings}. Let $s_x$ be a non-constant site output and
$(m_1,m_2)\ne(0,0)$ a mode in its support. Rotation invariance gives $m_1+m_2=0$ at every site. If
$x=a_j$, periodicity in $\beta$ gives $m_2\equiv0\pmod q$, so $|m_1|+|m_2|=2|m_2|\ge2q$; if
$x=b_k$, symmetrically $|m_1|+|m_2|\ge2p$; and if $x=c$, both periodicities apply and
$m_1=-m_2$ is a multiple of $\mathrm{lcm}(p,q)$. In all cases $|m_1|+|m_2|\ge2\min(p,q)$, and
Lemma~\ref{lg:lb:lem:budget} gives the first bound. For the readout \eqref{eq:Ei}--\eqref{eq:E},
the site output is a function of the invariant channels of $h^{(T)}(x)$, which have support
bounded by $\omega_T$ and are therefore all constant unless $\omega_T\ge2\min(p,q)$; a function
of constants is constant. For the consequences, Lemma~\ref{lg:lb:lem:rings} places
$\mathcal Y_{p,q}$ inside $\mathcal R_{\rho,p+q+1}$; take $p=q=(K-1)/2$ if $K$ is odd and
$p,q=(K-2)/2,\,K/2$ if $K$ is even, so that $2\min(p,q)=2\lfloor(K-1)/2\rfloor$.
\end{proof}

\begin{remark}[why the bound is linear]\label{lg:lb:rem:pooling}
Along $\mathcal Y_{p,q}$ the orbit invariant $\alpha-\beta$ is defined modulo
$2\pi/\mathrm{lcm}(p,q)$, and the output at the centre $c$ is accordingly constant unless its
frequency reaches $\mathrm{lcm}(p,q)$, quadratic in $K$ for $p,q$ coprime. This does not give a
quadratic bound on the model, because the outputs at the ring points are not subject to the
$\alpha$-periodicity individually, only as a multiset, and a nonlinear readout summed over a
ring that the symmetry permutes turns a frequency into a multiple of it: for coprime $p,q$,
\[
\sum_{j=1}^{p}\cos^{p}\bigl(q(\alpha-\beta)+\tfrac{2\pi qj}{p}\bigr)
=\tfrac{p}{2^{\,p-1}}\cos\bigl(pq(\alpha-\beta)\bigr)+\text{const}.
\]
The sum of per-point MLPs in \eqref{eq:Ei}--\eqref{eq:E} does exactly this, so the quadratic
bound of Theorem~\ref{thm:lower-quadratic} does not transfer to MACE; only the linear bound of
Theorem~\ref{lg:lb:thm:main}, which rests on the periodicity that does fix the ring points, does.
\end{remark}

\begin{corollary}\label{lg:lb:cor:bounded}
A local MACE model whose depth and body orders are bounded independently of $K$ cannot be
complete on $\mathcal R_{\rho,K}$ with an irreducible degree bounded independently of $K$: one
needs $\ellmax=\Omega(K)$.
\end{corollary}

\begin{remark}[the range of $K$]\label{lg:lb:rem:smallK}
Two rings of fewer than three points cannot be used in Definition~\ref{lg:lb:def:rings}: a ring
point must be $\rho$-adjacent to $c$, which requires four affinely independent points among the
points it sees, while its ring seen together with points of the axis spans a plane when the ring
has two points and a line when it has one; points of the other ring cannot be added without
making the distances depend on $\alpha-\beta$. This is why the family starts at $K=7$. On the
other hand, no counterexample can have $K\le3$: if $K\le3$, a ball of radius $\rho$ never
contains four affinely independent points of the cloud, so no two points are $\rho$-adjacent,
and \textnormal{(O)} forces the cloud to consist of a single point. Thus the only values of $K$ not
covered by the rings are $K=4,5,6$, and the following lemma covers them, and every larger $K$,
for the models that cannot produce a pseudoscalar.
\end{remark}

\begin{lemma}[parity]\label{lg:lb:lem:parity}
Consider MACE with the readout \eqref{eq:Ei}--\eqref{eq:E} and $T$ iterations, and suppose that
$\ellmax=0$, or $T=1$ and $\numax\le2$, or $T=2$ and $\numax=1$. Then every invariant channel
$h^{(T)}_{i,k00}$, and hence the output of the model, takes the same value on $X$ and on $-X$.
Consequently such a model is not $(\Gthree,\mathcal R_{\rho,K})$-complete on clouds of $K$ points
for any $K\ge4$.
\end{lemma}

\begin{proof}
Under $X\mapsto-X$ every relative vector changes sign and every distance is unchanged, so a
product of spherical harmonics $Y_{\ell_1}\cdots Y_{\ell_s}$ of relative directions, weighted by
radial functions, is multiplied by $(-1)^{\ell_1+\dots+\ell_s}$. Call a coordinate of type
$W_L$ \emph{natural} if it is multiplied by $(-1)^L$. The channel mixings in \eqref{eq:B},
\eqref{eq:pooling}, \eqref{eq:At} and \eqref{eq:update} are linear and preserve the type, so
they preserve naturality. If $\ellmax=0$ every feature is a function of distances and there is
nothing to prove. If $T=1$, the coordinates of $A^{(0)}$ are natural, and a coordinate of type
$W_0$ of $B^{(0)}$ is either a coordinate of $A^{(0)}$ of type $W_0$, when $\nu=1$, or, when
$\nu=2$, a Clebsch--Gordan contraction of two coordinates of the same type $W_\ell$, which is
multiplied by $(-1)^{2\ell}=1$; hence every $h^{(1)}_{i,k00}$ is inversion invariant. If $T=2$
and $\numax=1$, then $B^{(0)}=A^{(0)}$, so every coordinate of $h^{(1)}$ is natural; in
\eqref{eq:At} a coordinate of type $W_0$ of $A^{(1)}$ contracts $Y_{\ell_1}$ with a coordinate of
$h^{(1)}$ of type $W_{\ell_2}$, which requires $\ell_1=\ell_2$ and is therefore multiplied by
$(-1)^{2\ell_1}=1$; since $B^{(1)}=A^{(1)}$, every $h^{(2)}_{i,k00}$ is inversion invariant.

For the consequence, let $X$ consist of $K\ge4$ points, not all in a plane, with pairwise
distinct mutual distances all at most $\rho$. Then $X$ and $-X$ satisfy \textnormal{(D)},
\textnormal{(K)} and \textnormal{(O)}, every ball around a point containing all $K$ points, so
both lie in $\mathcal R_{\rho,K}$. If $g(-X)$ were a relabelling of $X$ for some $g\in SE(3)$, the
isometry $x\mapsto g(-x)$ would map $X$ onto itself preserving all distances; as these are
distinct it would fix every point of $X$, hence, fixing four affinely independent points, it
would be the identity, which is impossible since its linear part has determinant $-1$. So $X$
and $-X$ lie in different orbits and receive the same output.
\end{proof}

The hypotheses of Lemma~\ref{lg:lb:lem:parity} are sharp for MACE as defined in
Subsection~\ref{app:MACE_multi}: with $\ellmax\ge1$, a pseudoscalar arises for $T=1$ and
$\numax=3$ from the triple product $1\otimes1\otimes1\to0$, for $T=2$ and $\numax=2$ from the
cross product $1\otimes1\to1$ in the first iteration followed by the contraction with $Y_1$ in
\eqref{eq:At}, and for $T=3$ and $\numax=1$ from two consecutive couplings in \eqref{eq:At}. With the parity selection rule discussed in Subsection~\ref{app:MACE}, instead, every feature carries a parity $\pm1$, by which it is multiplied under $X\mapsto-X$, and the invariant channels have parity $+1$, so the conclusion of Lemma~\ref{lg:lb:lem:parity} holds for all $\ellmax$, $\numax$ and $T$.

\begin{remark}[small counterexamples in the literature]\label{lg:lb:rem:smalllit}
Known degeneracies of low body order are local and do not improve on these values. The
three-body descriptors of an atomic environment are complete for at most three neighbors and
have degenerate pairs as soon as there are four \citep{PhysRevLett.125.166001}, and
degeneracies of the four-body descriptors are known for environments of seven neighbors
\citep{Nigam_2024}. These are statements about the environment of one atom, and they
are mitigated when, as in \eqref{eq:E}, the output sums contributions from every atom; the
three-body case is in any event subsumed by Lemma~\ref{lg:lb:lem:parity}, which already gives
$K=4$, the least possible value.
\end{remark}

\begin{corollary}\label{lg:lb:cor:small}
Let MACE with the readout \eqref{eq:Ei} be $(\Gthree,\mathcal R_{\rho,K})$-complete on clouds of
$K$ points. Then
\[
\ellmax\ \ge\ \frac{2\lfloor(K-1)/2\rfloor}{\numax}\ \ (T=1),\qquad
\ellmax\ \ge\ \frac{2\lfloor(K-1)/2\rfloor}{\numax(1+\numax)}\ \ (T=2),
\]
\[
\ellmax\ \ge\ \frac{2\lfloor(K-1)/2\rfloor}{\numax(1+\numax+\numax^2)}\ \ (T=3).
\]
Equivalently, combining Theorem~\ref{lg:lb:thm:main} with Lemma~\ref{lg:lb:lem:parity}, the model
fails to be complete as soon as $K\ge4$ if it satisfies the hypotheses of
Lemma~\ref{lg:lb:lem:parity}, and otherwise as soon as
\[
K\ \ge\ \max\Bigl\{7,\ 2\Bigl\lfloor\tfrac12\,\ellmax\textstyle\sum_{i=1}^{T}\numax^{\,i}\Bigr\rfloor+3\Bigr\};
\]
these thresholds are listed in Table~\ref{tab:lb:thresholds} for $0\le\ellmax,\numax\le4$.
\end{corollary}

\begin{table}[h]
\centering
\small
\setlength{\tabcolsep}{3pt}
\begin{tabular}{c|cccc}
\hline
$\ellmax\,\backslash\,\numax$ & $1$ & $2$ & $3$ & $4$\\
\hline
$1$ & $4^{\dagger}/4^{\ast}/\mathbf{7}$ & $4^{\dagger}/\mathbf{9}/\mathbf{17}$ & $\mathbf{7}/\mathbf{15}/\mathbf{41}$ & $\mathbf{7}/\mathbf{23}/\mathbf{87}$\\
$2$ & $4^{\dagger}/4^{\ast}/\mathbf{9}$ & $4^{\dagger}/\mathbf{15}/\mathbf{31}$ & $\mathbf{9}/\mathbf{27}/\mathbf{81}$ & $\mathbf{11}/\mathbf{43}/\mathbf{171}$\\
$3$ & $4^{\dagger}/4^{\ast}/\mathbf{11}$ & $4^{\dagger}/\mathbf{21}/\mathbf{45}$ & $\mathbf{11}/\mathbf{39}/\mathbf{119}$ & $\mathbf{15}/\mathbf{63}/\mathbf{255}$\\
$4$ & $4^{\dagger}/4^{\ast}/\mathbf{15}$ & $4^{\dagger}/\mathbf{27}/\mathbf{59}$ & $\mathbf{15}/\mathbf{51}/\mathbf{159}$ & $\mathbf{19}/\mathbf{83}/\mathbf{339}$\\
\hline
\end{tabular}
\caption{Smallest number $K$ of points per cutoff ball from which local MACE with the readout
\eqref{eq:Ei}--\eqref{eq:E}, irreducible degree $\ellmax$ (rows), correlation order $\numax$
(columns) and $T$ iterations is provably not $(\Gthree,\mathcal R_{\rho,K})$-complete, by
Corollary~\ref{lg:lb:cor:small}. Each cell lists the thresholds for \textbf{$\mathbf{T=1}$\,/\,$\mathbf{T=2}$\,/\,$\mathbf{T=3}$}, in
this order. All entries equal to $4$ come from Lemma~\ref{lg:lb:lem:parity}: the invariant
channels of the model are unchanged under $X\mapsto-X$, so it cannot tell a chiral cloud from its
mirror image; $4$ is the least possible value, since $\mathcal R_{\rho,K}$ contains only one-point
clouds when $K\le3$ (Remark~\ref{lg:lb:rem:smallK}). \textbf{Entries marked $4^{\dagger}$ were known
before this work}: with $T=1$ the invariant channels are functions of distances ($\numax=1$) or
three-body invariants ($\numax=2$), which are classically inversion invariant. \textbf{Entries marked
$4^{\ast}$ rest on the same known mechanism}, but the observation that two
iterations with $\numax=1$ still produce no pseudoscalar is stated here. \textbf{Entries in bold are
new}: they come from the two-ring family of Definition~\ref{lg:lb:def:rings} and
Theorem~\ref{lg:lb:thm:main}, apply to models that can produce pseudoscalars, and grow linearly in
$\ellmax$ and like $\numax^{T}$ in $\numax$.}
\label{tab:lb:thresholds}
\end{table}

\begin{remark}\label{lg:lb:rem:prior}
Lower bounds of this kind in the learning literature are for one fixed order at a time: the
counterexamples of \citet{Pozdnyakov_2022} exhibit pairs of clouds that a given body order cannot
tell apart. The rings add a bound that grows with the size of the neighbourhood, by the
mechanism of Theorem~\ref{thm:lower-quadratic} with the azimuthal frequency in the role of the
polynomial degree. Theorem~\ref{thm:maceLocal} spends an irreducible degree quadratic in $K$ on
one iteration; whether this can be brought down to linear is open.
\end{remark}

\section{Invariant Theory Background}\label{app:invariant_theory}
Suppose that $G$ is a group that acts a vector space $V$.  In Invariant Theory one studies polynomial functions on $V$ that are constant on orbits. The origins of Invariant Theory go back to the $19^{\rm th}$ century. Historically, one often works over $\C$, but many results have been extended to arbitrary fields, including fields with positive characteristic. For our applications we will work over $\R$. Suppose
 $V$ is an $n$-dimensional  $\R$-vector space.
 We denote the ring of all polynomial functions $V\to \R$ by $\R[V]$. If $x_1,x_2,\dots,x_n$ are linearly independent coordinate functions on $V$, then 
 we can identify $V$ with $\R^n$ and 
 $\R[V]$ is just the polynomial ring $\R[x_1,x_2,\dots,x_n]$.
 Suppose that a group $G$ acts on $V$.
 For $v\in V$, let $[v]=\{g\cdot v\mid g\in G\}$ be the $G$-orbit of $v$. A polynomial invariant is a polynomial function $V\to \R$ that is constant on orbits. The  invariant ring $\R[V]^G$ is the set of all polynomial invariants.

 Suppose that $G$ is a  Lie group and $\rho:G\to \GL(V)$ is an orthogonal representation, i.e., $\rho$ is a continuous homomorphism of groups. 
  Then $G$ acts on $V$ via $g\cdot v=\rho(g)v$ for all $v\in V$ and $g\in G$.
  If $G$ is compact, then one can choose a basis of $V$ such that $\rho(G)$ is contained in the orthogonal group $\OO(n)$.
 \begin{theorem}[\citealp{hilbert1890ueber}]\label{theo:Hilbert}
     If $G$ is compact then invariant ring $\R[V]^G$ is generated by finitely many invariants $f_1,f_2,\dots,f_r$. So every invariant polynomial is of the form $p(f_1,f_2,\dots,f_r)$ where $p\in \R[y_1,y_2,\dots,y_r]$ is a polynomial with coefficients in $\R$.
 \end{theorem}
 If the group $G$ is not compact, then the invariant ring may not be finitely generated. The following well-known result can be found in  \cite[Ch. 3, §4, 3°, Thm. 3]{OnishchikVinberg1990}:
 \begin{theorem}\label{theo:InvariantsSeparate}
     Suppose that $G$ is compact, $u,v\in V$ and $[u]\neq [v]$. Then there exists an invariant $f\in \C[V]^G$ with $f(u)\neq f(v)$. If $f_1,f_2,\dots,f_r$ are the generators from Theorem~\ref{theo:Hilbert}, then $f_i(u)\neq f_i(v)$ for some $i$.
 \end{theorem}
We next recollect the definition of completeness from the main text:
\begin{definition}
A set ${\mathcal Q}\subseteq \R[V]^G$ is called $(G,V)$-complete if
    $$
    \forall q\in {\mathcal Q}\ q(u)=q(v)\quad\Leftrightarrow\quad [u]=[v].$$
\end{definition}
Theorem~\ref{theo:InvariantsSeparate} states the sets $\C[V]^G$ and the set $\{f_1,f_2,\dots,f_r\}$ of generates are both $(G,V)$-complete.
 
If $G$ is not compact, then a $(G,V)$-complete set of invariants may not exist. Take for example the action of $\GL(V)$ on $V$ where $n\geq 1$. There are two orbits in $V$, namely $\{0\}$ and $V\setminus\{0\}$. Every polynomial invariant is constant on $V\setminus\{0\}$ and by continuity also constant on $V$. So no polynomial invariant can separate the two orbits. To take into account this situation, we recall the general notion of separating invariants in Invariant Theory:
\begin{definition} 
A set ${\mathcal Q}\subseteq \R[V]^G$ is called  separating if for all $u,v\in V$ we have
    $$\forall q\in {\mathcal Q}\ q(u)=q(v)\quad\Leftrightarrow\quad \forall f\in \R[V]^G\ f(u)=f(v).$$
\end{definition}
For compact $G$, the notion of completeness and separation are equivalent.

Tautologically, $\R[V]^G$ is separating, but it may not be finitely generated. However, we have:
 \begin{theorem}
     There exists a finite set of separating invariants.
 \end{theorem}
 A proof is given in~\cite[Theorem 2.4.8]{derksen2015computational} (the assumption that the base field is algebraically closed is not necessary there), or see~\cite[Theorem 12]{Kemper2009} for a more general statement. Note that if the set $\R[V]^G$ of all invariants is $(G,V)$-complete, then any set of separating invariants is $(G,V)$-complete. In particular, there exists a finite $(G,V)$-complete set of invariants in that case.

The separating degree $D(G,V)$ is the smallest nonnegative integer $d$ such that the space $K[V]^G_{\leq d}$ of invariants of degree $\leq d$ is a separating set for the action of $G$ on $V$. In this manuscript we work over $K=\RR$. Complexification is introduced below when discussing bounds on the degrees
of generators.

If any $(G,V)$-complete set of invariant polynomials exists, then $D(G,V)$ is also the smallest nonnegative integer for which  $K[V]^G_{\leq d}$ is $(G,V)$-complete. We have $D(G,V)\leq \beta(G,V)$ where $\beta(G,V)$ is the smallest nonnegative integer $d$ such that $K[V]^G$ is generated by $K[V]^G_{\leq d}$.
Upper bounds for $\beta(G,V)$ have been studied extensively. If $G$ is a compact Lie group and $\rho:G\to \GL(n,\R)$ is a real representation, then $\rho$ extends to a complex representation $\rho_{\C}:G_{\C}\to \GL(n,\C)$, where $G_{\C}$ is the complexification of $G$. Now $G_{\C}$ is a complex reductive group. We have  $\beta(G,\R^n)=\beta(G_{\C},\C^n)$ and upper bounds are discussed in~\cite[\S3.2, \S3.10, \S4.7]{derksen2015computational}. Emmy Noether proved the bound $\beta(G,V)\leq |G|$ for finite groups $G$ (\citealp{noether1916endlichkeitssatz}). For permutation actions one has upper bounds for $\beta(G,V)$ that often are sharper.
Suppose that $G$ acts on $V=\R^n$ by permuting the basis vectors $e_1,e_2,\dots,e_n$ of $V=\R^n$. 
In this case we have
\begin{equation}\label{eq:Gobel}
\beta(G,V)\leq \max\left\{n_1,n_2,\dots,n_r,{n_1\choose 2}+{n_2\choose 2}+\cdots+{n_r\choose 2}\right\}\leq \max\left\{n,{n\choose 2}\right\},
\end{equation}
where $n_1,n_2,\dots,n_r$ are the cardinalities of the orbits for the action of $G$ on $\{e_1,e_2,\dots,e_n\}$. 
This follows for example from G\"obel's algorithm (see~\citealp{gobel1995computing}). 

For weighted graphs with $n$ vertices we consider the action of the symmetric group $S_n$ on the space $\Symn\subseteq \R^{n\times n}$ of symmetric $n\times n$ matrices by conjugation.
Thi\'ery \citeyearpar{thiery2000algebraic} gave the following upper bound
$$
D(S_n,\Symn)\leq \beta(S_n,\Symn)\leq  {{n\choose 2}\choose 2}-\mu_n=\frac{n^4}{8}
+O(n^3),
$$
where $\mu_n=0$ if $n$ is even and $\mu_n=\lceil \frac{3}{4}(n-1)\rceil$ if $n$ is odd. This bound is slightly sharper than the one obtained from ~(\ref{eq:Gobel}) with $n_1={n\choose 2}$ and $n_2=n$. In this paper we  improve the upper bound  for $D(S_n,\Symn)$ to 
$$
D(S_n,\Symn)\leq \frac{5n^3}{2}-n^2-\frac{n}{2}=O(n^3)
$$
and also provide a lower bound proportional to $n^2$. 

\end{document}